\documentclass[letterpaper]{article} 
\usepackage{aaai24}   
\usepackage{times}  
\usepackage{helvet}  
\usepackage{courier}  
\usepackage[hyphens]{url}  
\usepackage{graphicx} 
\usepackage{natbib}  
\usepackage{caption} 
\usepackage{algorithm}
\usepackage{algorithmic}

\usepackage{newfloat}
\usepackage{listings}
\DeclareCaptionStyle{ruled}{labelfont=normalfont,labelsep=colon,strut=off} 
\floatstyle{ruled}
\newfloat{listing}{tb}{lst}{}
\floatname{listing}{Listing}

\newif\ifapp
\apptrue

\usepackage[utf8]{inputenc} 
\usepackage{url}            
\usepackage{nicefrac}       
\usepackage{amsfonts,amsmath,mathtools,amsthm}
\usepackage{newtxmath}
\usepackage{microtype}
\usepackage{subfig}
\graphicspath{{figs/}}
\newtheorem{thm}{Theorem}
\newtheorem{lem}[thm]{Lemma}
\newtheorem{cor}[thm]{Corollary}

\theoremstyle{definition}

\newtheorem{definition}[thm]{Definition}
\newtheorem{assumption}[thm]{Assumption}
\newcommand{\eq}[1]{\begin{footnotesize}\begin{align}#1\end{align}\end{footnotesize}}

\newcommand{\nin}{\not\in}
\newcommand{\teb}[1]{\phantom{abc}\text{#1}}

\ifapp
\title{Understanding and Leveraging the Learning Phases of Neural Networks\footnote{This is the extended version with all proofs and additional datasets of the accepted paper at AAAI 2024 in the end.}}
\else
\title{Understanding and Leveraging the Learning Phases of Neural Networks}
\fi
\author {
    Johannes Schneider\textsuperscript{\rm 1}\footnote{Work conducted during holidays.},
    Mohit Prabhushankar \textsuperscript{\rm 2}
}
\affiliations {
    \textsuperscript{\rm 1}University of Liechtenstein, Vaduz, Liechtenstein\\
    \textsuperscript{\rm 2}Georgia Institute of Technology, Atlanta, USA\\
    johannes.schneider@uni.li, mohit.p@gatech.edu
}

\begin{document}

\maketitle

\begin{abstract}
The learning dynamics of deep neural networks are not well understood. The information bottleneck (IB) theory proclaimed separate fitting and compression phases. But they have since been heavily debated. We comprehensively analyze the learning dynamics by investigating a layer's reconstruction ability of the input and prediction performance based on the evolution of parameters during training. We empirically show the existence of three phases using common datasets and architectures such as ResNet and VGG: (i) near constant reconstruction loss, (ii) decrease, and (iii) increase. We also derive an empirically grounded data model and prove the existence of phases for single-layer networks. Technically, our approach leverages classical complexity analysis. It differs from IB by relying on measuring reconstruction loss rather than information theoretic measures to relate information of intermediate layers and inputs. Our work implies a new best practice for transfer learning: We show empirically that the pre-training of a classifier should stop well before its performance is optimal. 
\end{abstract} 

\section{Introduction}
Deep neural networks are arguably the key driver of the current boom in artificial intelligence(AI) in academia and industry. They achieve superior performance in a variety of domains. Still, they suffer from poor understanding, leading to an entire branch of research, i.e., explainability(XAI)~\cite{mesk22}, and to widespread debates on trust in AI within society. Thus, enhancing our understanding of how deep neural networks work is arguably one of key problems in ongoing machine learning research~\cite{pog20}. Unfortunately, the relatively few theoretical findings and reasonings are often subject to rich controversy. 

One debate surrounds the core of machine learning: learning behavior. ~\citet{tish17} leveraged the information bottleneck(IB) framework to investigate the learning dynamics of neural networks. IB relies on measuring mutual information between activations of a hidden layer and the input as well as the output. A key qualitative, experimental finding was the existence of a fitting and compression phase during the training process. Compression is conjectured a reason for good generalization performance. It is frequently discussed in the literature~\cite{gei21,jak19}. Such a finding can be considered a breakthrough in understanding deep neural networks. However, its validity has been disputed, i.e., ~\citet{saxe19} claimed that statements by~\citet{tish17} do not generalize to common activation functions. Today, the debate is still ongoing~\cite{lor21}. A key challenge is approximating the IB. This, makes a rigorous mathematical treatment very hard -- even empirical analysis is non-trivial.

\begin{figure}[!htb]
\vspace{-8pt}
\centering{\centerline{\includegraphics[width=0.45\textwidth]{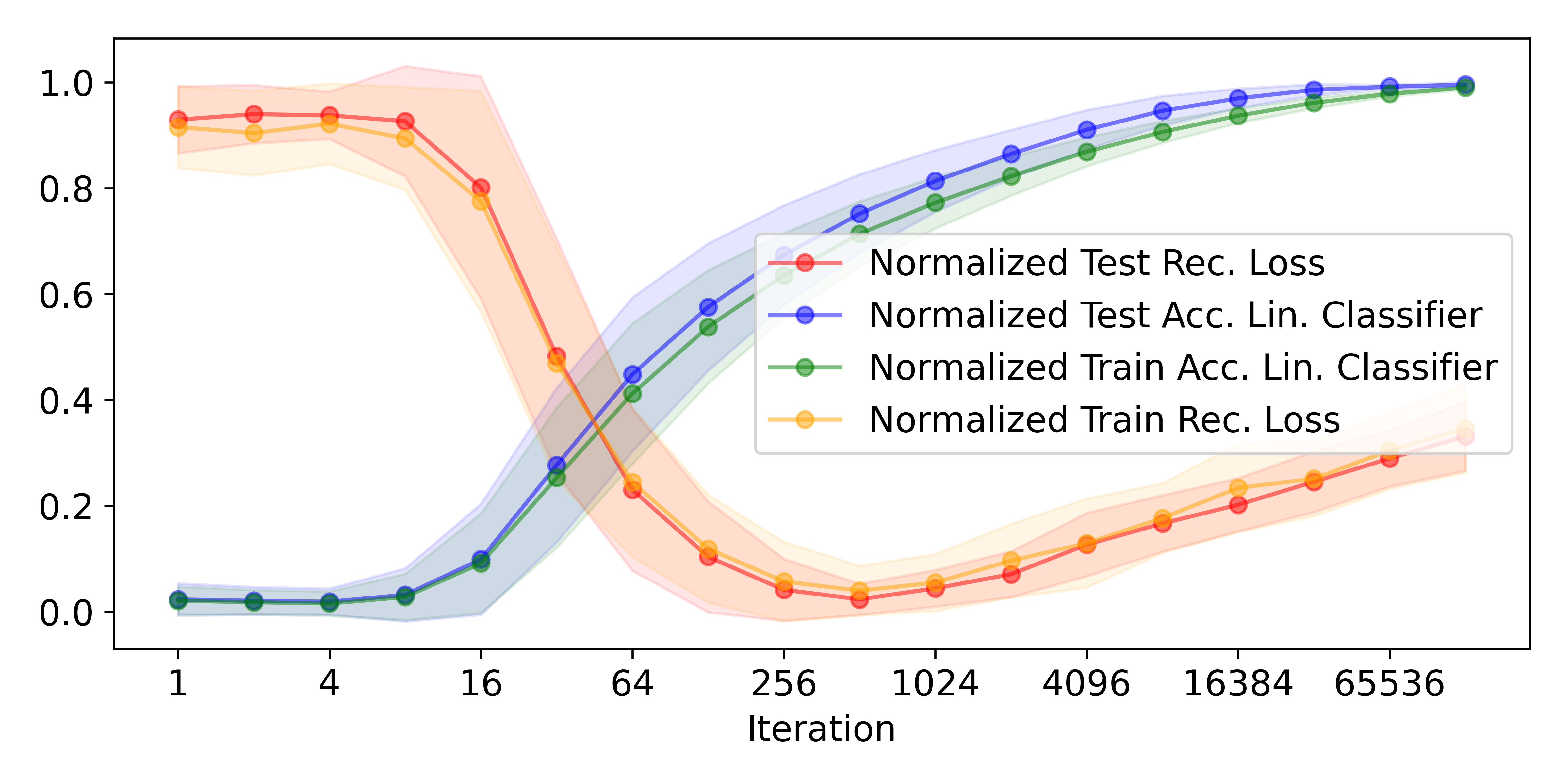}} 
\vspace{-8pt}
\caption{Normalized accuracy and reconstruction loss for a linear classifier and the FashionMNIST dataset} \label{fig:metF0Fa}}
\vspace{-8pt}
\end{figure}

In this work, we study the learning behavior with a similar focus, i.e., classification and reconstruction capability of layers. We propose a different lens for investigation rooted in classical complexity analysis that leads to more precise statements. We perform a rigorous analysis of a linear classifier and a data model built on empirical evidence. 
First, to the best of our knowledge prior work relies on general statements that lack mathematical proof. Instead, we show bounds on the duration of phases, also highlighting interplays among data characteristics, e.g., between the number of samples, the number of classes and the number of input attributes.
Second, we utilize different measures from IB. IB measures information of a layer with respect to input and output~\cite{tish17}. We measure (i) accuracy, e.g., how well the network classifies samples\footnote{In the full version we also train a separate classifier (as in~\cite{ala17}) rather than relying on network performance. We train a simple classifier on the layer activations, i.e., we look at linear separability of classes given layer activations}, and (ii) the reconstruction error of the input given the layer by utilizing a decoder.
Third, we provide a data model and prove for single-layer networks, i.e., linear networks, that it can cause the three learning phases, i.e., the reconstruction error is likely to decrease early in training after a phase of near constant behavior and increases later during training.   

As practical implication, we show that pre-training of classifiers that are later fine-tuned should stop well before the performance is optimal for the pre-training task. That is, we argue and show that it is decisive how much information on the original dataset is kept.

The paper begins with an empirical analysis showing that the alleged phases can be observed for multiple classifiers, datasets and layers followed by a theoretical analysis drawing on our empirical findings to model the problem and rigorously analyze a linear multi-class model. Finally, we elaborate on transfer learning, state related work, discuss our work, and conclude. 

\section{Empirical analysis}
For a model $M=(L_0,L_1,...)$ consisting of a sequence of layers $L_i$, we investigate the reconstruction loss of inputs of a decoder $DE^{(t)}$ trained for each iteration $t$. The decoder is trained to output a reconstruction $\hat{x}$ from a layer activation $L_i(x)$, i.e., $\hat{X}=DE(L_i(x))$. The reconstruction error $Rec^{(t)}$ is $||\hat{x}-x||^2$. We compare this error against the model's accuracy $Acc^{(t)}$. (In the full version, we also discuss accuracy of a linear classifier $CL$ trained on layer activations $L_i$.)




\subsection{Datasets, Networks and Setup} \label{sec:dat} 
As networks $M$ we used VGG-11~\cite{sim14}, Resnet-10~\cite{he16} and a fully connected network, i.e., we employed networks $F0$, which equals the theoretical setup, i.e., one dense layer followed by a softmax activation. After each hidden layer we applied the ReLU activation and batch-normalization. We used a fixed learning rate of 0.002 and stochastic gradient descent with batches of size 128 training for 256 epochs.  

We computed evaluation metrics $Acc^{(t)}$ and $Rec^{(t)}$ at iterations $2^i$. 
For the decoder $DE$ we used the same decoder architecture as in~\cite{sch21cla}, where a decoder from a (standard) auto-encoder was used.  For each computation of the metrics, we trained the decoder for 30 epochs using the Adam optimizer with a learning rate of 0.0003.
We reconstructed from the network outputs, i.e., the last dense layer (index -1), the second last layer (index -2), i.e., the one prior to the (last) dense layer and layer with index -3, which indicates using outputs of the second last conv layer for VGG and the second last block for ResNet.\footnote{We don't show reconstructions from softmax outputs as they follow the pattern even more strongly.} 
We used CIFAR-10/100~\cite{kri09}, Fashion-MNIST~\cite{xia17}, and MNIST~\cite{den12}, all scaled to 32x32, available under the MIT (first 3 datasets) and GNU 3.0 license. We trained each model $M$ five times. All figures show standard deviations. We report normalized metrics to better compare $Acc^{(t)}$ and $Rec^{(t)}$.

\subsection{Observations}


Figure \ref{fig:metF0Fa} shows results for FashionMNIST using a linear classifier $F0$. Figure \ref{fig:metMuFa} shows the outputs for a ResNet for multiple layers for the MNIST and FashionMNIST datasets. 
\ifapp
Additional datasets and classifiers are in the end.
\fi
The behavior of all three classifiers across datasets is qualitatively identical.  As expected, accuracy increases throughout training. The reconstruction loss remains stable for the first few iterations before decreasing and increasing towards the end, highlighting the existence of multiple phases. For layers closer to the input, phases become less pronounced and reconstruction loss overall is less. That is, the phases are well-observable when normalizing each line (left panel in \ref{fig:metMuFa}, but less so for lower layers as seen in the right panel, where lines are not normalized separately. Put differently, the closer to the output, the more easily the phases are observable. For layers close to the input they are not noticeable. This is aligned with existing knowledge that lower layers in classifiers converge faster and are more generic\cite{zeil14}, i.e., fit less to input data (especially, its labels). Thus, empirically, we have observed the existence of the phases, which we investigate more profoundly in our theoretical analysis.


\begin{figure*}
\vspace{-8pt}
  \centering
  
  \subfloat[ResNet-10 (each line normalized separately)]{\includegraphics[width=0.45\textwidth]{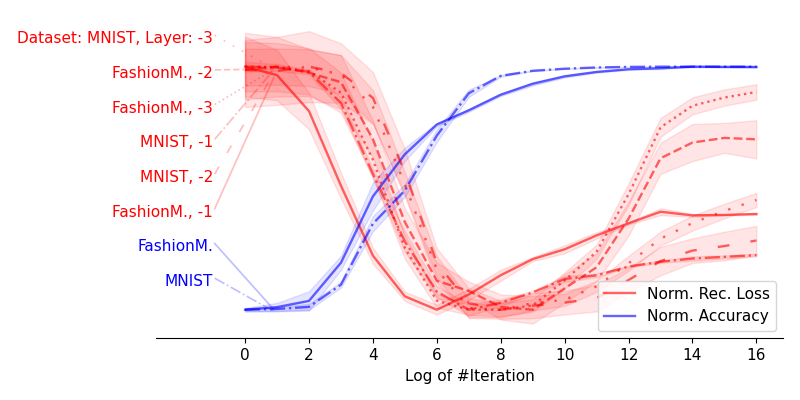} \label{fig:b1}} 
  \subfloat[ResNet-10 ]{\includegraphics[width=0.45\textwidth]{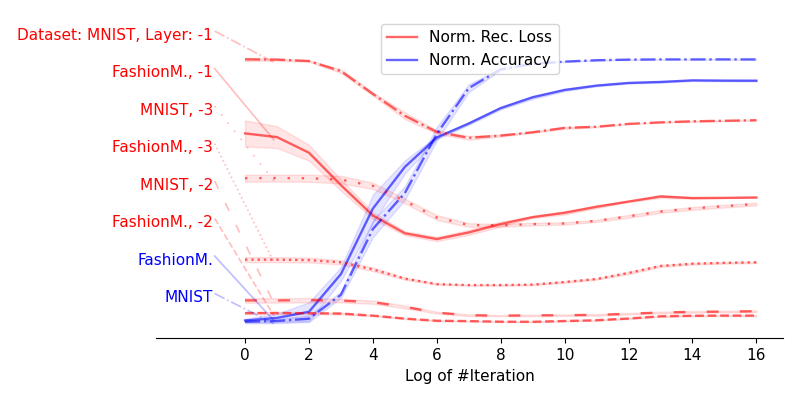} \label{fig:a1}} 
  \vspace{-4pt}
  \ifapp
  \caption{Accuracy and reconstruction loss for a ResNet for the FashionMNIST and the MNIST dataset. (Negative) layer indices indicate skipped layers from the output as described in text.  Other datasets and classifiers are in the end.} \label{fig:metMuFa}
  \fi
  \vspace{-8pt}
\end{figure*}

\section{Theoretical analysis} 
We analyze the reconstruction loss over time of a linear decoder taking outputs of a simple classifier as input. 
We follow standard complexity analysis from computer science using $O$-notation deriving bounds regarding the number of samples $n$, dimension $d$ of the inputs, and number of classes $l$. As common, we assume the quantities in the bounds such as $n$, $d$, and $l$ are large, allowing us to discard lower order terms in $n$, $d$, and $l$ and constants. 

\subsection{Definitions and Assumptions}
\begin{figure}
\vspace{-4pt}
  \centering
  \includegraphics[width=0.85\linewidth]{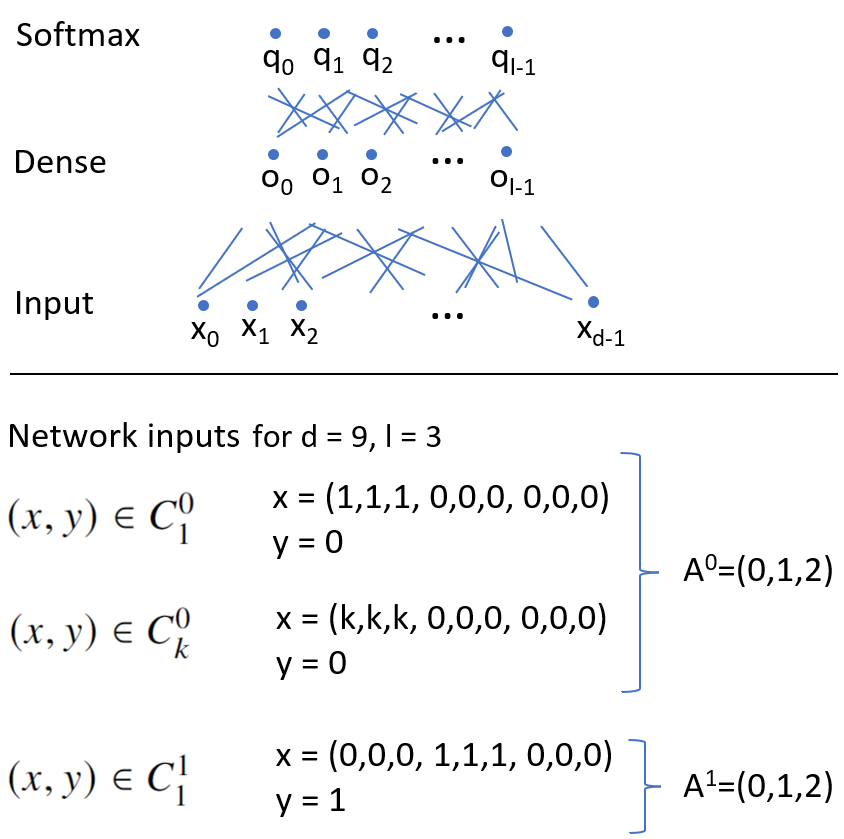}
\vspace{-4pt}
\caption{Illustration of network with all to all connections and inputs} \label{fig:netin}
  \vspace{-4pt}
\end{figure}
\noindent\textbf{Data:} We consider a labeled dataset $D=\{(x,y)\}$ consisting of pairs $(x,y)$ with $d$-dimensional input $x=(x_0,x_1,..,x_{d-1})$ and label $y \in [0,l-1]$, i.e., we have $l$ classes and each input has $d$ attributes. We denote $n =|D|$ as the number of samples. The set $C^y=\{\big(x |(x,y') \in D\big) \wedge \big(y'=y\big)\}$ comprises of all inputs of class $y$. We assume balanced classes, i.e., $|C^i|=|C^{j}|=n/l$. We denote the set of indexes for all $d$ input attributes as $A:=(0,1,...,d-1)$. For a sample $(x,y) \in D$ of a class $y$, only the subset of attributes $A_y:=[y\cdot d/l,(y+1)\cdot d/l] \subset A$ are non-zero. This builds on the assumption that features are only present for at most a few classes. 

Our data builds on the natural assumption that some samples are easier and others are more difficult to recognize, i.e., input features have different strengths. That is, some samples might appear like prototypical samples with well-notable characteristics of that class, i.e., \emph{strong features} and others might appear more ambiguous, i.e., have \emph{weak features}.
Rather than using a continuous distribution for feature strengths, we employ a discrete approximation, where a feature strength is either 0, 1 (weak) or $k$ (strong). That is, feature strengths differ by a factor $k$. We assume $k\geq 2$.\footnote{$k>1$ suffices with a more complex analysis.} We treat it as a constant. That is, although we do not subsume it in our asymptotic analysis, it should be seen as a constant, i.e., $\Theta(k)=\Theta(1)$.
We say that samples $C^y$ of class $y$ can be split into two equal-sized disjoint subsets, i.e., \emph{weak samples} with only \emph{weak features} $C^y_{1}$, and \emph{strong samples} with only \emph{strong features} $C^y_{k}$.\footnote{Samples having a mix of both do not change outcomes, but add to notational complexity}. 

The data samples $(x,y)$ are defined as:
\eq{ 
x:= \label{def:dat}
\begin{cases}
  x_{i \in A_{y}}=1,  & \text{if}\ x \in  C^y_1 \text{ 
 (weak features for weak samples) }\\ 
  x_{i \in A_{y}}=k,  & \text{if}\ x \in  C^y_k \text{ 
 (strong features for strong samples) }\\ 
  x_{i}=0  & \text{otherwise (non-present features for both)}
\end{cases} 
}

The data are illustrated in Figure \ref{fig:netin}. In our data model, the value of an attribute (or feature) is zero for most samples. It is non-zero, i.e., either 1 or $k$, for samples of a single class. Also, the input attributes do not change their values throughout training. Real data, i.e., inputs to the last layer throughout training, is shown for comparison in Figure \ref{fig:indistr}. For real data, the input distribution of one layer evolves during training as weights of lower layers change. For VGG-16, inputs stem from a ReLU layer and, thus, the majority of outputs is 0 as proclaimed in our model.  For Resnet-10, the inputs stem from a 4x4 average pooling, which averages outcomes of a ReLU layer. Thus, there is no strong peak at 0, but rather inputs are small positive values. Despite the averaging after a few iterations, well-noticeable, class-dependent differences between attributes emerge. Thus, the essence of our model is also captured for Resnet: Most input attributes of (samples of) a class are small. An input attribute is only large for samples of one or a few classes.

\begin{figure*}
  \vspace{-4pt}
  \subfloat[Data model]{\includegraphics[width=0.42\linewidth]{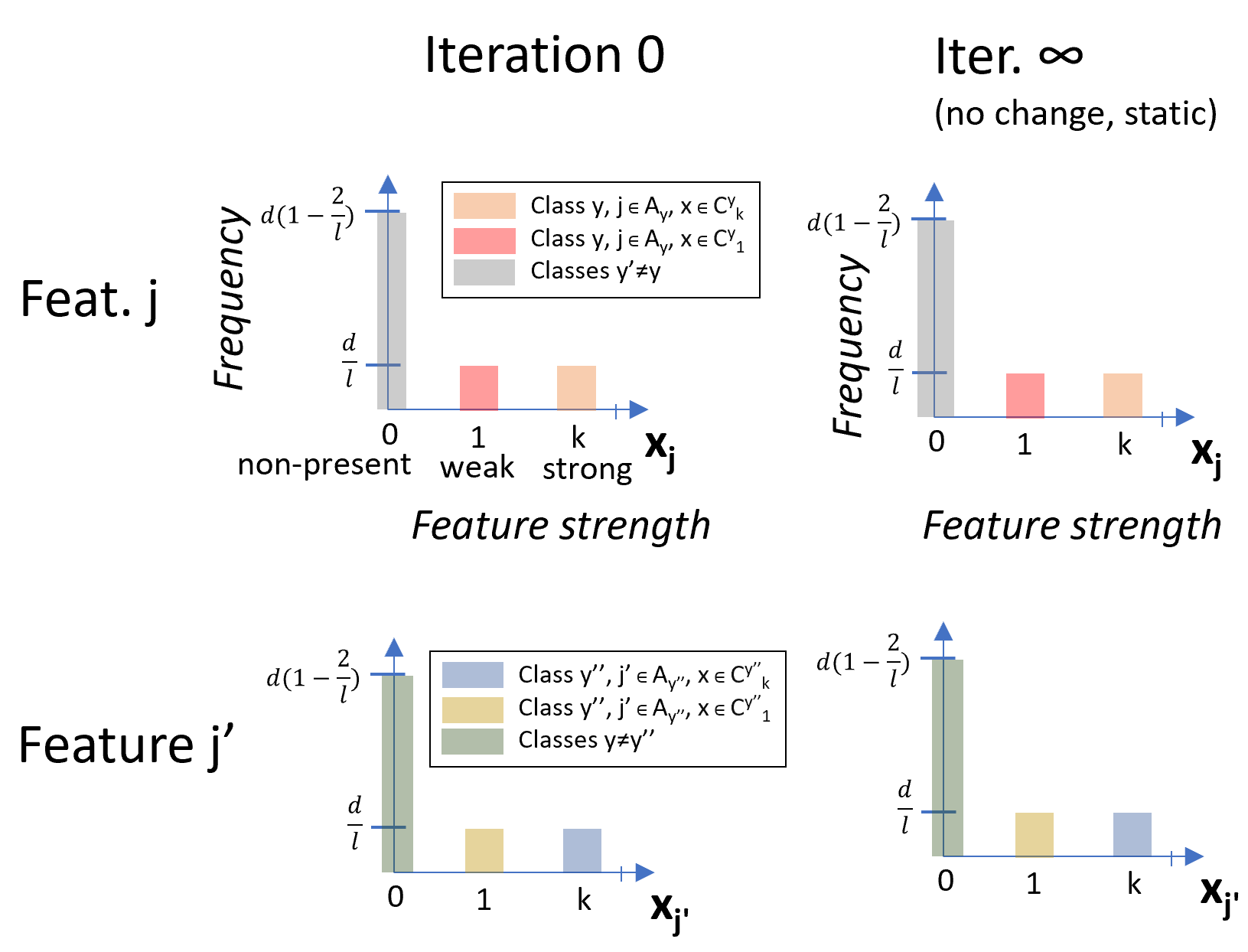}}
  \quad 
  \subfloat[Real data: ResNet-10 on FashionMNIST.]{\includegraphics[width=0.27\linewidth]{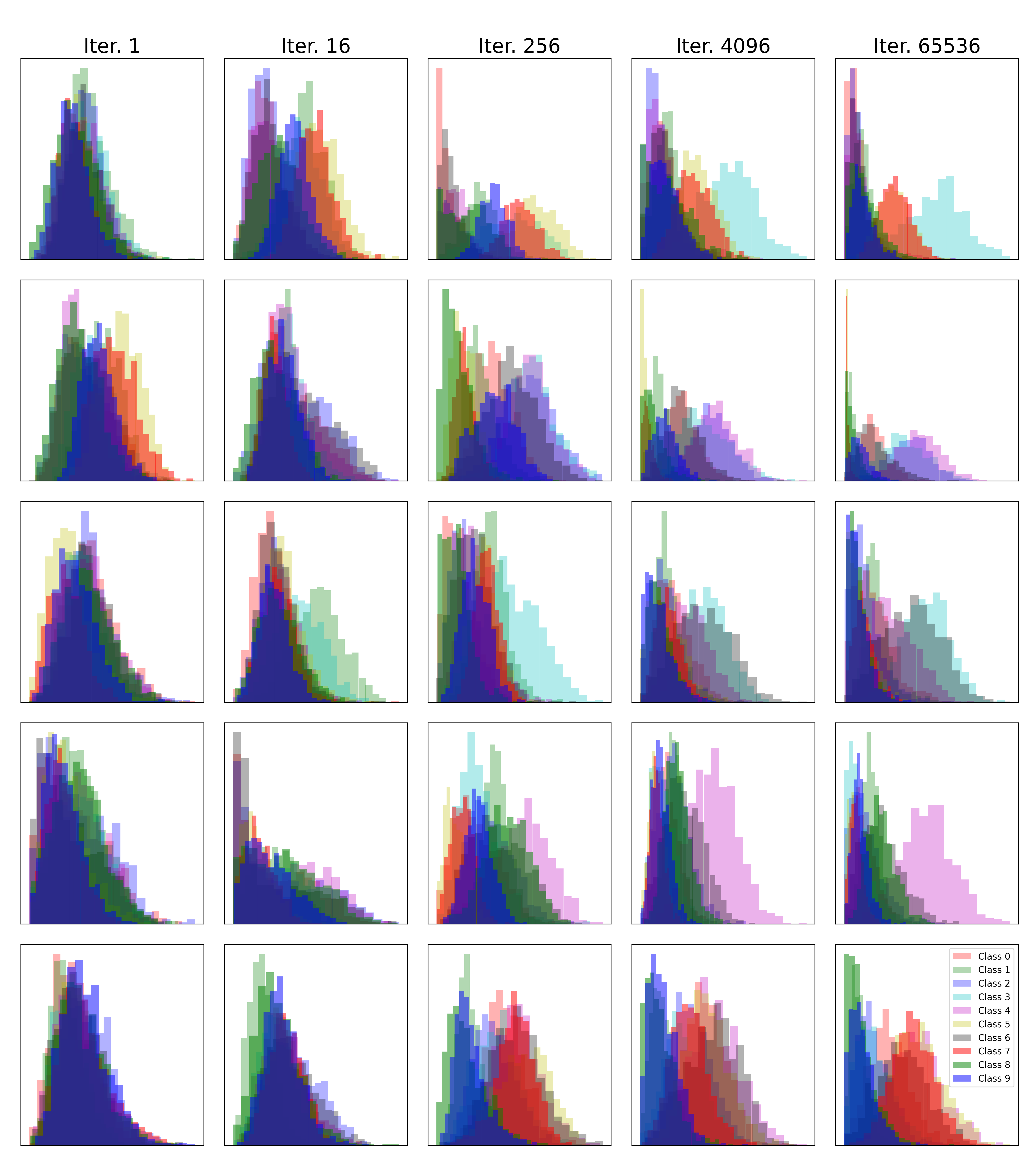}}
  \quad
  \subfloat[Real data: VGG-16 on CIFAR-10]{\includegraphics[width=0.27\textwidth]{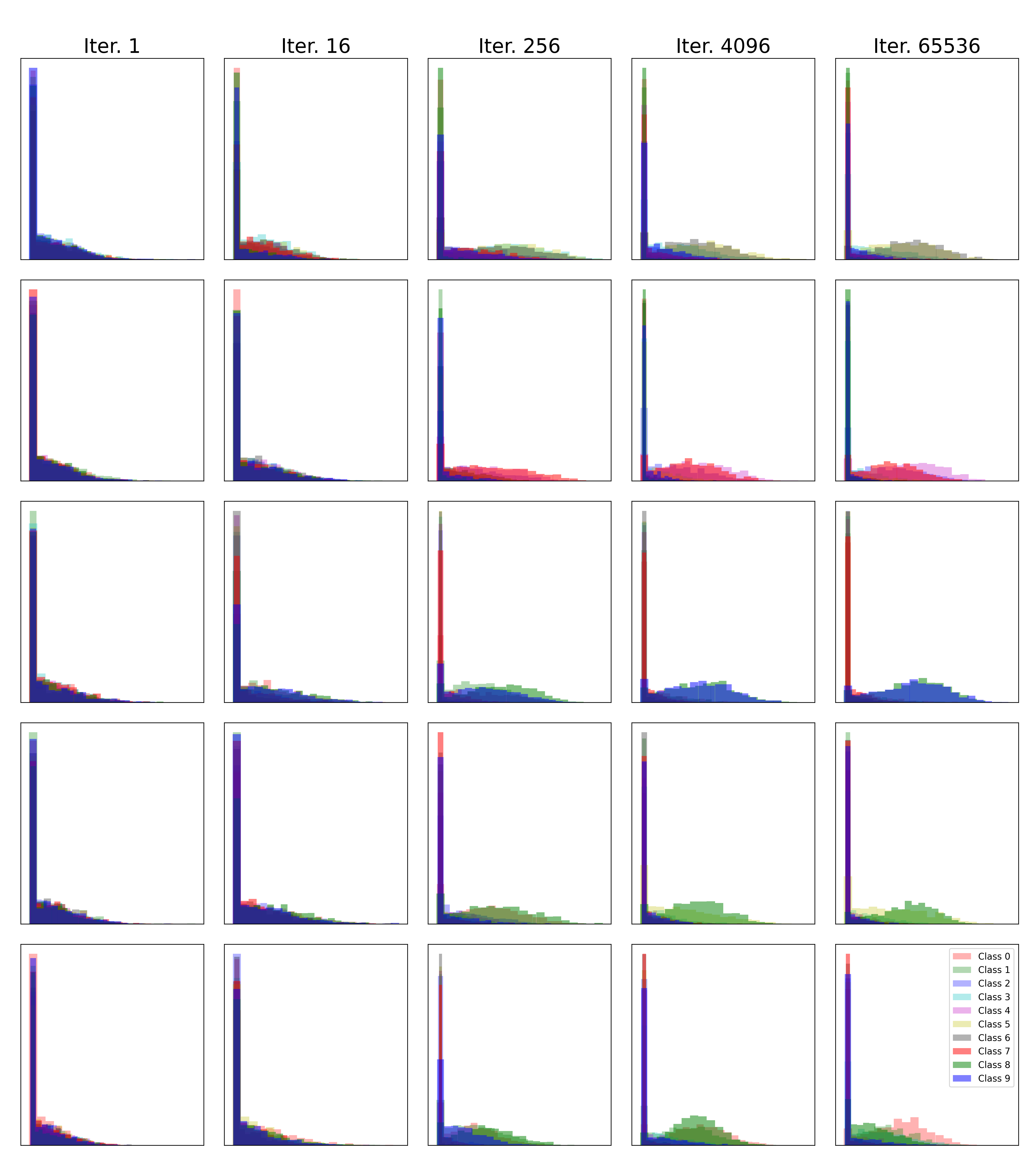} \label{fig:b6}} 
  \vspace{-4pt}
\caption{Class-dependent distribution of an input attribute fed into the last dense layer. Rows show features, columns iterations. Initially, inputs are equally distributed for all classes, but already after a few iterations samples of one or a few classes have larger values. ResNet is less concentrated due to 4x4 average pooling. } \label{fig:indistr}
\end{figure*}

\noindent\textbf{Network:}
The network is illustrated in Figure \ref{fig:netin}. All parameters (and parameter-dependent entities such as outputs) are time-dependent, i.e., with the superscript $^{(t)}$. For readability, we omit the superscript in expressions, where all entities share the same time $t$.

The network output stems from a single dense layer followed by a softmax activation.
\begin{definition}[Linear Layer Output] \label{def:out}
The output for a sample $x$ is $o(x)=(o_0(x),o_1(x),...,o_{l-1}(x))$. The scalar $o_y$ is the output for class $y$ defined as $o_y:=o_y(x):=w_y\cdot x= \sum_{i<d} w_{y,i}\cdot x_i$.
\end{definition} %

The softmax function to compute the class probability for a class $y$ given the output $o(x)$ for a sample $(x,y'')$ is: 
\eq{
q(y|x)=\frac{e^{o_y}}{\sum_{y'<l}e^{o_{y'}}} \label{def:q}
}
If all elements $x$ in a set $C$ are identical we write as an abbreviation:
\eq{
q(y|C):= q(y|x \in C)\\
o(C):= o(x \in C)
}

\begin{assumption}[Weight Initialization] \label{ass:initw}
$w^{(0)}_{y,j} \sim N(0,1/d)$ (random), $w^{(0)}_{y,j} = 1/d$ (deterministic)
\end{assumption}
For random initialization, initial noise becomes much smaller in magnitude relative to the changes in weights during training. This diminishes its impact over time.  Thus, we first assume deterministic initialization for simplicity and analyze random initialization in the extended version. The random initialization  follows common initialization schemes~\cite{he15,sch22}.

\noindent\textbf{Cross-entropy Loss and Optimization:}
We employ the cross-entropy loss for a sample $(x,y)$ defined as $L(x,y)=\sum_{y'<l}-\vmathbb{1}_{y=y'}\cdot\log(q(y'|x))$. The indicator variable $\vmathbb{1}_{cond}$ evaluates to $1$ if the condition $cond$ holds and otherwise to 0. We perform gradient descent. The update of weight $w^{(t)}_{y,j}$ at iteration $t$ is 
\begin{align}
w^{(t+1)}_{y,j}&=w^{(t)}_{y,j}-\frac{\lambda}{|D|} \sum_{(x,y') \in D}  \nabla_{w_{y,j}} L(x,y') \label{eq:grad}
\end{align}
A linear classifier can perfectly classify the data (Def. \ref{def:dat}): Weights $w_{y,i}$ for $i \in A_y$ should be very large since the loss tends to 0 as these weights tend towards $\infty$. All other weights should be very small since the loss tends to 0 as these weights tend towards -$\infty$.

\medskip
\noindent\textbf{Reconstruction:}
A linear reconstruction (or decoding) function $g_j$ is defined for each input attribute $j$.
\begin{definition}[Reconstruction Function] \label{def:brecF}
The reconstruction function $g_j$ takes as input the output vector $o(x)$ for a sample $(x,y)$ and learnable parameters are a bias $b$ and a slope $s_{y'}$ for each class $y'$, i.e., each scalar $o_{y'}$ in vector $o(x)$:
\begin{align}
g_{j}(o(x)):=b_j+ \sum_{y'<l} s_{y',j}\cdot o_{y'}
\end{align}
\end{definition}
Note, like the parameters $w^{(t)}$ of the classifier, the parameters $b^{(t)}$ and $s^{(t)}$ of the reconstruction function are not fixed throughout training. They are fit for each iteration $t$. That is, $g_j^{(t)}$ also depends on time. 
The reconstruction loss for an attribute $x_j$ of a sample $(x,y)$ is the squared difference between $x$ and the reconstructed inputs $g_j$. The total reconstruction loss $R^{(t)}$ is just the average of the individual losses. 
\eq{
R^{(t)}:= \sum_{(x,y) \in D} \frac{({x}_j-g^{(t)}_j(o(x)))^2}{|D|} \label{eq:recraw}
}

\ifapp
\subsection{Analysis}
We put all theorems with proofs at the very end of the extended version.
\fi
We formally show the existence of two phases, i.e., that the reconstruction loss decreases and then increases again. To this end, we split the analysis into three stages based on sums of weights $f$, which are linked to iterations $t$. We show that the error decreases in Stage 1, starts to increase again in Stage 2 and stabilizes in Stage 3. Due to symmetry, it suffices to focus on one class $y$ and aim to fit four (weighted) points as shown in Figure \ref{fig:linvis}. In fact, we use an approximate reconstruction function and show that it is not far from the optimal. Our approach relies on heavy calculus, in particular (Taylor) series expansions, and classical O-notation to keep the complexity of the analysis manageable. 

To get a flavor, we derive an expression for the change of weights in each iteration, which is further expanded during the analysis. To this end, we leverage symmetries arising from the data definition and the deterministic initialization. 
We compute derivatives for weights $w_{y,j}$ for a sample $(x,y')$ separately for the class $y=y'$ and $y\neq y'$, and for samples of different strengths, i.e., $(x,y') \in C^y_{v}$ with $v \in \{1,k\}$.
The derivative $\frac{\partial L(x,y)}{\partial w_{y',j}}$ of the loss with respect to network parameters is non-zero if inputs are non-zero, i.e., $x_j\neq 0$, which only holds for $j\in A_{y}$ (Def. \ref{def:dat} and Figure \ref{fig:netin}). For derivative of loss with a softmax holds (see Section 5.10 in~\cite{jur22}):
\begin{small}
\eq{ 
\frac{\partial L(x,y)}{\partial w_{y',j}}=\frac{\partial L(x,y)}{\partial o_{y'}} \cdot \frac{\partial o_{y'}}{\partial w_{{y'},j}}=(q_{y'|x}-\vmathbb{1}_{y'=y})\cdot x_j\\
=\begin{cases}
v(q(y'|x)-\vmathbb{1}_{y=y'}) & \text{if}\ j\in A_{y}, x \in C^{y}_{v} \\ 
 0  & \text{otherwise} \label{eq:dLy}
\end{cases}
}
\end{small}

Next, we rephrase the reconstruction loss (Eq. \ref{eq:recraw}), which sums the error across all samples.
Due to symmetry, it suffices to focus on one attribute $x_j$ with $j \in A_y$ for some $y$. There are also only four different reconstruction errors due to symmetry: 
\begin{enumerate}
    \item  $R_1$ states the error for reconstructing $x_j=1$ with $j \in A_y$ given the output for a sample $x \in C^y_1$
    \item  $R_k$ for $x_j=k$ with $x \in C^y_k$
    \item    $R^{(t)}_{0,C^{y'}_1}$ for $x_j=0$ with $x \in C^{y\neq y'}_1$
    \item $R^{(t)}_{0,C^{y'}_k}$ for $x_j=0$ with $x \in C^{y\neq y'}_k$    
\end{enumerate}

Formally, these errors and the total error $R$ are given by:

\eq{
&R^{(t)}_{v,C}:= \sum_{x \in C, x_j=v} \frac{({x}_j-g^{(t)}_j(o(x)))^2}{|C|} \\
&R^{(t)}_{1}:=R^{(t)}_{1,C^y_1} \teb{\phantom{a}} R^{(t)}_{k}:=R^{(t)}_{k,C^y_k} \teb{ (Abbreviation)}\\
&R^{(t)}_{0,C^{y'}_1}:=R^{(t)}_{1,C^{y'\neq y}_1} \teb{\phantom{a}} R^{(t)}_{0,C^{y'}_k}:=R^{(t)}_{k,C^{y'\neq y}_k} \teb{ (Abbreviation)}\\
&R^{(t)}=\big((1-\frac{1}{2l})\cdot(R^{(t)}_{0,C^{y'}_1}+R^{(t)}_{0,C^{y'}_k})+\frac{1}{l}\cdot(R^{(t)}_{1}+R^{(t)}_{k})\big) \label{def:brecLoss} } 
That is, the total reconstruction error of an attribute $j \in A_y$ is the weighted sum of errors for $x_j=0$, which are most prevalent with a weight of $1-\frac{1}{2l}$, and the reconstruction errors for $x_j=1$ and $x_j=k$, which are less common, i.e., only a fraction $\frac{1}{l}$ of samples has $x_j=1$ or $x_j=k$ for $j \in A_y$. 

We aim to reconstruct input attributes $x_j \in \{0,1,k\}$ from output probabilities $q^{(t)}(y'|x)$ optimally for each time $t$. We choose non-optimal but simpler reconstruction functions depending on $t$ and bound the error due to the approximation. We show that these functions approximate the true reconstruction error asymptotically optimally. As approximation, we fit the reconstruction function $g_j$ using two of the four output values $q$, i.e., the outputs for $x_j=k$ with $j\in A_y$ for $x\in C^{y}_k$ and for $x_j=0$ for $x\in C^{y'\neq y}_k$  (see Figure \ref{fig:linvis} and the four errors listed in Def. \ref{def:brecLoss}). These two points are fitted optimally, i.e., without error, while the other two can have larger errors than the optimal reconstruction function $g^{opt}_j$. The motivation for the selection of the two specific points is as follows. We use $x_j=0$ since most attributes $x_j$ are 0, i.e., out of $d$ attributes only a fraction $2/l$ are non-zero. Therefore, as $d$ is assumed to be large, even small errors in reconstructing attributes $x_j=0$ can yield large overall errors. The choice of $x_j=k$ rather than $x_j=1$ is to have a larger spread between points used to fit the reconstruction function.

For $t>0$, we use:
\eq{
g_j(o)=\frac{k}{q(y|C^y_k)-q(y|C^{y'\neq y}_k)}\cdot (o_j-q(y|C^{y'\neq y}_k))  \label{eq:gj}
}
Put differently, geometrically, we fit a line $g_j$ without error through two points given as tuple (output $q$,value $x_j$ to reconstruct), i.e., $(q(y|C^{y'\neq y}_k),0)$ and $(q(y|C^y_k),k)$ as illustrated in Figure \ref{fig:linvis}. Thus, it follows from the construction of $g_j$ that the reconstruction errors $R_k$ (i.e., for $x_j=k$ and $x \in  C^y_k$) and  $R_{0,C^{y'}_k}$ (i.e., for $x_j=0$ and $x \in C^{y'\neq y}_k)$ are both zero.

For $t=0$, where outputs $q$ are identical for all inputs, we use a constant reconstruction function being a weighted average:
\eq{&g^{(t)}_j(o)=\frac{\frac{n}{l} + \frac{n}{l} k+ n(1-\frac{1}{2l})\cdot 0}{n}= \frac{1+k}{l}}

\medskip

\begin{figure*}
\vspace{-8pt}
  \centering
  \includegraphics[width=\linewidth]{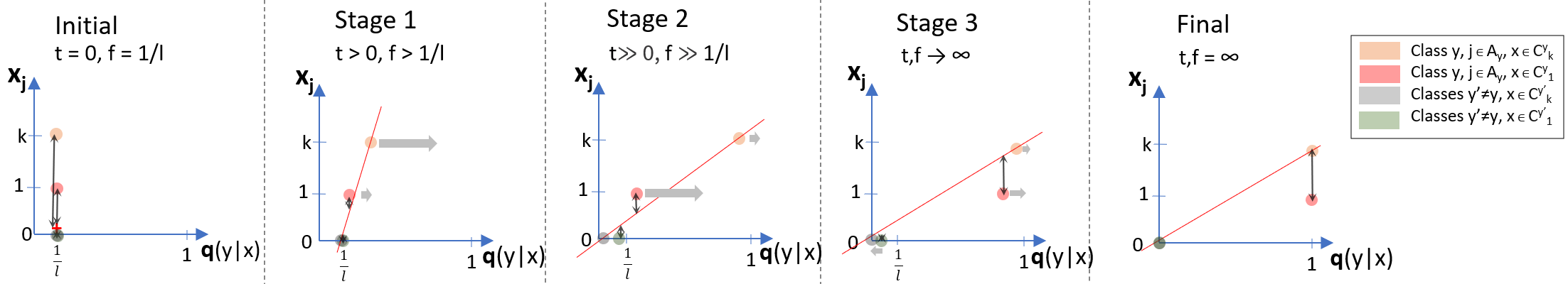}
\vspace{-8pt}
\caption{Stages distinguished for analysis. Stages are separated based on sums of weights ($f$), which are linked to iterations $t$. Each panel shows network outputs $q$ (blue points) versus attributes $x_i$ for samples $(x,y) \in D$ and the (approximate) reconstruction function of $x_i$ from $q$ (red line) over time $t$. The reconstruction  approximates the optimal reconstruction. Vertical bars indicate error bars. Grey horizontal arrows shows how outputs $q$ change compared to the panel on the left. } \label{fig:linvis}
\vspace{-8pt}
\end{figure*}

\medskip

\noindent \textbf{Reconstruction Error}: Next, we bound the error of the reconstruction function. We use different stages suitable for analysis. We first  express them not using time $t$ but in terms of sums of weights. 
Figure \ref{fig:linvis} illustrates the stages: 
Initially, all probabilities $q$ are equal and the reconstruction error is large. During the first stage $q(y|C^y_k)$ increases rapidly for (strong) samples with strong features, reducing the initial reconstruction error. In the second phase, $q(y|C^y_k)$ (for strong samples) changes much less, while $q(y|C^y_1)$ for (weak) samples grows fast and catches up. Within this stage the reconstruction error increases again. In the third stage both $q(y|C^y_1)$ and $q(y|C^y_k)$ have roughly the same magnitude and converge towards 1. The reconstruction error still slowly increases but also converges. The fact that differences between the two diminish, worsens the reconstruction as shown in the rightmost panel in Figure \ref{fig:linvis}, since the best reconstruction is in the middle of both well-separated points rather than being close to each of them.  



We formally derive each reconstruction error for each stage separately. The derivation for each error and stage follows the same schema. We simplify the reconstruction errors $R_1$ and $R_{0,C^{y'\neq y}_1}$ using series expansions. We do not derive a single expression for a reconstruction error for all iterations $t$ but rather split the analysis by looking at intervals of weight values, i.e., their sums $f$. These are then linked to iterations $t$. Considering intervals constraining the sum of weights allows to further simplify expressions. Still, a significant amount of calculus is required to obtain closed-form expressions.
We obtain bounds for each individual reconstruction error (see Eq. \ref{def:brecLoss}), combining all individual reconstruction errors yielding the total error based on our approximate reconstruction function $g_j$: 
\begin{cor}
The reconstruction error for $g_j$ decreases from $R^{(0)}=\Omega(k^2/l)$ to $R^{(t)}=O(k^2/l^2)$ with $t \in [1,\Theta(\frac{2l(c_f-1)}{\lambda d(k+1)})]$ for an arbitrary constant $c_f>1$ and increases to $R^{(t\rightarrow\infty)}=\Omega(k^2/l)$.
\end{cor}

But our proclaimed reconstruction function $g_j$ is not optimal. Thus, finally, we need to bound the deviation of our reconstruction error based on $g_j$ from the optimal reconstruction error. This can be best understood based on our illustration Figure \ref{fig:linvis}, where the optimal reconstruction $g^{opt}$ line might not go through two points exactly as $g_j$ does, but rather be more in the ``middle'' of all points. 

\begin{thm}\label{thm:recErr}  
For the approximation error of the reconstruction function $g_j$ of the optimal function $g^{opt}$ holds $||g_j-g^{opt}_j||<2$ .
\end{thm}
As proof strategy, we consider the individual terms $R_i$ that sum to the total reconstruction error $R$. Each $R_i$ represents the error of one set of points to reconstruct. We show that changing the points for reconstruction allows reducing the most dominant error $R_i$ only up to a constant factor before another error $R_j$ becomes dominant. Thus, asymptotically our approximation is optimal.

\section{Improved Transfer Learning} \label{sec:tra}
Our work shows that while the discriminative performance of networks increases, their ability to accurately describe the data (as measured by reconstruction ability) decreases. Thus, if a classifier is used for fine-tuning on a dataset $D'$ or as feature extractor that should describe the data well, it might be better to stop training of the classifier $C$ on the (original, large) dataset $D$ before the cross-entropy loss is minimal, i.e., before the classifier performance is maximized on $D$. When exactly to stop depends on the similarity of the datasets $D$ and $D'$, i.e., it might not be necessarily when the reconstruction loss is minimal. \\
To assess this hypothesis, we proceed analogously as for reconstruction (see Section \ref{sec:dat} for details). We train a classifier $C$, i.e., VGG or Resnet, on a dataset $D$. We then freeze the classifier $C$ and train a linear classifier on a dataset $D'$ taking as input the output of a layer (last (-1), second last(-2), etc.) of the classifier $C$. More precisely, for dataset $D$ being FashionMNIST, we use $D'$ being MNIST and for $D$ being CIFAR-10 we use CIFAR-100 as $D'$. We also assess the scenarios with $D$ and $D'$ switched. Furthermore, we also assess to predict, which color channel, i.e., `red',`green',`blue', has largest average value used across all pixels. Figures \ref{fig:trans} and \ref{fig:trans2} show two exemplary outputs. It can be seen that accuracy on the downstream have a clear maximum, which tends to be roughly after the same iterations when the reconstruction loss is minimal. Thus, we see that some maintaining more ``information'' on the original dataset (as observed due to lower reconstruction loss) is helpful for downstream task, since these tasks might exactly rely on this information.

\begin{figure}[!htb]
\centering{\centerline{\includegraphics[width=0.45\textwidth]{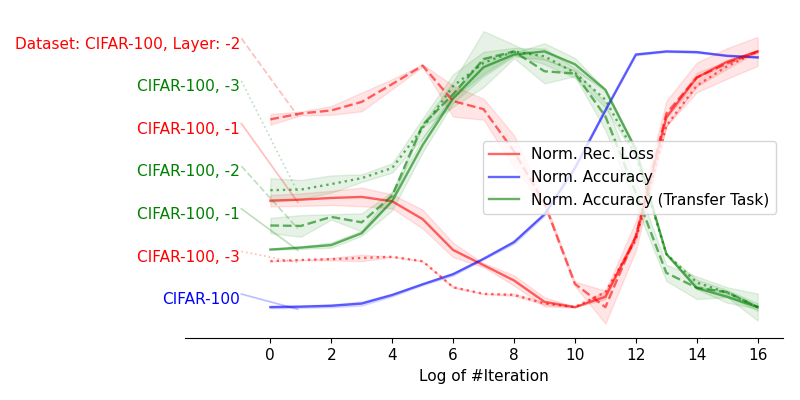}}} 
\vspace{-6pt}
\caption{Transfer Learning of a Resnet from Cifar-100 to Cifar-10. Lines are normalized.}\label{fig:trans}  
\vspace{-6pt}
\end{figure}
\begin{figure}[!htb]
\centering{\centerline{\includegraphics[width=0.45\textwidth]{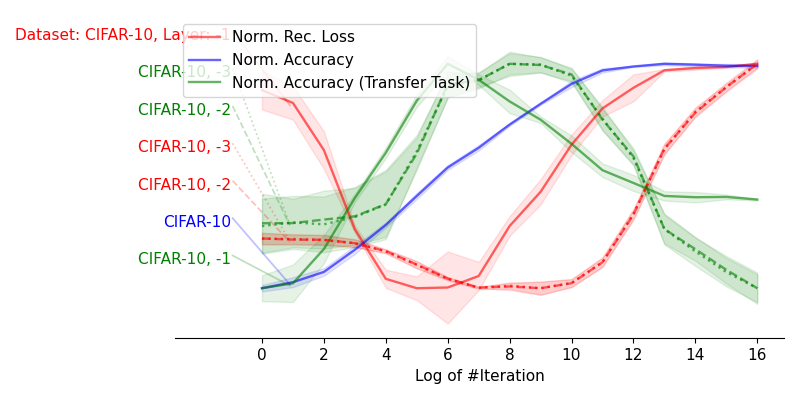}}} 
\vspace{-6pt}
\caption{Transfer Learning of VGG from predict Cifar-10 classes to predict the most dominant color channel. Lines are normalized.}\label{fig:trans2}
\vspace{-6pt}
\end{figure}

Figure \ref{fig:trans} shows the normalized performance of the original classifier, the reconstruction loss and the accuracy for the fine-tuning task.

\section{Discussion} \label{sec:gen}
We have shown empirically and analytically the existence of different phases for the reconstruction error. 
Our analysis is fairly different to prior works and, as such, interesting in its own right. That is, we aim to focus on approximating the governing optimization equations, which allows to capture the reconstruction dynamics over time. An essential part of our work are the underlying assumptions on which the theoretical analysis builds. While we verified our assumptions with empirical analysis (Figure \ref{fig:indistr}), any assumption is a limitation. Furthermore, our analysis relies on shallow networks, i.e., we essentially analyze the last dense layer and the softmax layer, which already requires multiple pages of calculus. Approaches as described in \cite{mai19} could help in extending it to deeper networks, but it is unclear, whether the analysis would reamin tractable. On the positive side, we derive concise bounds including multiple network parameters and covering also the number of iterations. This is often foregone in other works that focus on learning dynamics, e.g., they might require that width of layers tend to infinity to approximate layers with distribution\cite{jaco18,hua20}. 

While we have shown empirically  that the three phases of the reconstruction loss can occur for multiple datasets, networks, and layers, they are most notable for the last layers of a network. Features in the last layers are the least shared among classes, since the softmax layer paired with the cross-entropy loss of the classifier enforces discriminative features, i.e., features that strongly correlate with one or few classes only. That is, during training outputs of the last linear layer become more and more discriminative, forgoing information on the sample. The softmax also reduces the nuanced differences between classes. Thus, information content on the appearance is lost or difficult to retrieve in upper layers. This effect has also been observed in \cite{saxe19} using mutual information through simulations rather than mathematical proofs. 
This behavior can also be observed for other loss functions like the hinge loss as shown empirically in the extended version.
We anticipate that these phases are also prevalent in other architectures like RNNs and transformers, since the key assumptions such as the existence of weak samples (weak feature strenghts) and strong samples is not architecture dependent. 

\section{Related Work}
The information bottleneck~\cite{tish00} was used for analysis of deep learning ~\cite{tish15,tish17}. It suggests a principled way to ``find a maximally compressed mapping of the input variable that preserves as much as possible the information on the output variable''~\cite{tish17}. To this end, layers $h_i$ of a network are viewed as a Markov chain for which holds given $i\geq j$ using the data processing inequality: $I(Y;X) \geq I(Y;h_j) \geq I(Y;h_i) \geq I(Y;\hat{Y})$
Learning is seen as the process that maximizes $I(Y;h_i)$ while minimizing $I(h_{i-1};h_i)$, where the latter can be interpreted as the minimal description length of the layer. In our view, we agree on the former (at least on a qualitative level), but we do not see minimizing the description length as a goal of learning. In our perspective, it can be a consequence of the first objective, i.e. to discriminate among classes, and existing learning algorithms, i.e., gradient descent. From a generalization perspective, it seems preferable to cling onto even the smallest bit of information of the input $X$, even if it is highly redundant and, as long as it \emph{could} be useful for classification. This statement is also supported by~\cite{saxe19} who show that compression is not necessary for generalization behavior and that fitting and compression happen in parallel rather than sequentially. The review~\cite{gei21} also concludes that the absence of compression is more likely to hold. In contrast to our work, their analysis is within the IB framework. Still, it remedies an assumption of~\citet{tish17}. Namely,~\citet{saxe19} investigates different non-linearities, i.e., the more common ReLU activations rather than sigmoid activations. ~\citet{lor21} argues that compression is only observed consistently in the output layer. The IB framework has also been used to show that neural networks must lose information~\cite{liu20} irrespective of the data it is trained on. From our perspective, the alleged information loss measured in terms of the reconstruction capability is a fact though it can be small. In particular, it is evident that at least initially reconstruction is almost perfect for wide networks following theory on random projections, i.e., the Johnson-Lindenstrauss Lemma~\cite{joh84} proved that random projections allow embedding $n$ points into an $O(\log n/\epsilon^2)$ dimensional space while preserving distances within a factor of $1\pm \epsilon$. This bound is also tight according to~\citet{lar17}, and it can be extended to cases where we apply non-linearities, i.e., $ReLU$. \\
\citet{wang21pac} used an information measure based on weights to show empirically that fitting and compression phase exist. \citet{ach18} suggest two phases based on empirical analysis. They call the second phase ``forgetting''.

~\citet{gei21} also discussed the idea of geometric compression based on prior works on IB analysis. However, the literature was inconclusive according to~\citet{gei21} on whether compression occurs due to scaling or clustering. Our analysis is inherently geometry (rather than information) focused.\\ \citet{ala17} used a single linear layer to analyze networks empirically. However, their focus was to understand the suitability of intermediate layers for classification rather than learning dynamics. As the IB has also been applied to other types of tasks, i.e., autoencoding~\cite{tap20}, we believe that our approach might also be extended to such tasks.
The idea to reconstruct inputs from layer activations has been outlined in the context of explainability~\cite{sch21cla,sch22exp}. The idea is to compare reconstructions using a decoder with original inputs to assess what information (or concepts) are ``maintained'' in a model. Our work also touches upon linear decoders that have been studied extensively, e.g.,~\cite{kun19}. It also estimates reconstruction errors from noisy inputs $x_i+\epsilon$~\cite{car09}.

Dynamics of learning have been studied using Neural Tangent Kernel (NTK) \cite{jaco18,hua20}. NTK focuses on infinite-width networks by showing that they are equal to Gaussian processes. For example, it has been shown that fully connected networks of a particular size show linear rate convergence towards zero training error (Corollary 2.5 in \cite{hua20}). However, they do not discuss the relationship between reconstruction and classification, which is the focus of our work.

The impact of loss functions on transfer learning was studied in \cite{kor21}. Aligned with our work it was found that models performing better on the pre-training task can perform worse on downstream tasks.
Diversity of features\cite{nay22} has also been shown to lead to better downstream performance. This is also aligned with our work, since more diversity is also likely meaning that more information is captured, i.e., reconstructions are better. In contrast to these works, we proclaim that it is essential how much information on the input is maintained.  

\section{Conclusions}
Theory of deep learning is limited. This work focused on a very pressing problem, i.e., understanding the learning process. To this end, it rigorously analyzed a simple dataset modeling empirical observations of common datasets and classifiers. Our results highlight the existence of multiple phases for the reconstruction loss. This insight can be used to improve transfer-learning using early stopping of pre-training. 

\bibliography{refs}

\begin{thebibliography}{34}
\providecommand{\natexlab}[1]{#1}

\bibitem[{Achille, Rovere, and Soatto(2018)}]{ach18}
Achille, A.; Rovere, M.; and Soatto, S. 2018.
\newblock Critical learning periods in deep networks.
\newblock In \emph{International Conference on Learning Representations}.

\bibitem[{Alain and Bengio(2017)}]{ala17}
Alain, G.; and Bengio, Y. 2017.
\newblock Understanding intermediate layers using linear classifier probes.
\newblock In \emph{International Conference on Learning Representations {ICLR}
  (Workshop)}.

\bibitem[{Carroll, Delaigle, and Hall(2009)}]{car09}
Carroll, R.~J.; Delaigle, A.; and Hall, P. 2009.
\newblock Nonparametric prediction in measurement error models.
\newblock \emph{Journal of the American Statistical Association}, 104(487):
  993--1003.

\bibitem[{Deng(2012)}]{den12}
Deng, L. 2012.
\newblock The mnist database of handwritten digit images for machine learning
  research.
\newblock \emph{IEEE Signal Processing Magazine}, 29(6): 141--142.

\bibitem[{Geiger(2021)}]{gei21}
Geiger, B.~C. 2021.
\newblock On Information Plane Analyses of Neural Network Classifiers--A
  Review.
\newblock \emph{IEEE Transactions on Neural Networks and Learning Systems}.

\bibitem[{He et~al.(2015)He, Zhang, Ren, and Sun}]{he15}
He, K.; Zhang, X.; Ren, S.; and Sun, J. 2015.
\newblock Delving deep into rectifiers: Surpassing human-level performance on
  imagenet classification.
\newblock In \emph{Proc. of the international conference on computer vision},
  1026--1034.

\bibitem[{He et~al.(2016)He, Zhang, Ren, and Sun}]{he16}
He, K.; Zhang, X.; Ren, S.; and Sun, J. 2016.
\newblock Deep residual learning for image recognition.
\newblock In \emph{Conference on computer vision and pattern recognition
  (CVPR)}, 770--778.

\bibitem[{Huang and Yau(2020)}]{hua20}
Huang, J.; and Yau, H.-T. 2020.
\newblock {Dynamics of deep neural networks and neural tangent hierarchy}.
\newblock In \emph{International conference on machine learning}, 4542--4551.
  PMLR.

\bibitem[{Jacot, Gabriel, and Hongler(2018)}]{jaco18}
Jacot, A.; Gabriel, F.; and Hongler, C. 2018.
\newblock {Neural tangent kernel: Convergence and generalization in neural
  networks}.
\newblock \emph{Advances in neural information processing systems}, 31.

\bibitem[{Jakubovitz, Giryes, and Rodrigues(2019)}]{jak19}
Jakubovitz, D.; Giryes, R.; and Rodrigues, M.~R. 2019.
\newblock Generalization error in deep learning.
\newblock In \emph{Compressed Sensing and Its Applications}, 153--193.

\bibitem[{Johnson and Lindenstrauss(1984)}]{joh84}
Johnson, W.~B.; and Lindenstrauss, J. 1984.
\newblock Extensions of Lipschitz mappings into a Hilbert space.
\newblock \emph{Contemporary mathematics}, 26.

\bibitem[{Jurafsky and Martin(2021)}]{jur22}
Jurafsky, D.; and Martin, J.~H. 2021.
\newblock Speech and Language Processing.
\newblock \emph{Draft of 3rd edition}.

\bibitem[{Kornblith et~al.(2021)Kornblith, Chen, Lee, and Norouzi}]{kor21}
Kornblith, S.; Chen, T.; Lee, H.; and Norouzi, M. 2021.
\newblock {Why do better loss functions lead to less transferable features?}
\newblock \emph{Advances in Neural Information Processing Systems}, 34:
  28648--28662.

\bibitem[{Krizhevsky and Hinton(2009)}]{kri09}
Krizhevsky, A.; and Hinton, G. 2009.
\newblock Learning multiple layers of features from tiny images.
\newblock Technical report.

\bibitem[{Kunin et~al.(2019)Kunin, Bloom, Goeva, and Seed}]{kun19}
Kunin, D.; Bloom, J.; Goeva, A.; and Seed, C. 2019.
\newblock Loss landscapes of regularized linear autoencoders.
\newblock In \emph{International Conference on Machine Learning}, 3560--3569.

\bibitem[{Larsen and Nelson(2017)}]{lar17}
Larsen, K.~G.; and Nelson, J. 2017.
\newblock Optimality of the Johnson-Lindenstrauss lemma.
\newblock In \emph{2017 IEEE 58th Annual Symposium on Foundations of Computer
  Science (FOCS)}, 633--638. IEEE.

\bibitem[{Liu et~al.(2020)Liu, Qin, Anwar, Caldwell, and Gedeon}]{liu20}
Liu, Y.; Qin, Z.; Anwar, S.; Caldwell, S.; and Gedeon, T. 2020.
\newblock Are deep neural architectures losing information? invertibility is
  indispensable.
\newblock In \emph{International Conference on Neural Information Processing},
  172--184. Springer.

\bibitem[{Lorenzen, Igel, and Nielsen(2021)}]{lor21}
Lorenzen, S.~S.; Igel, C.; and Nielsen, M. 2021.
\newblock Information Bottleneck: Exact Analysis of (Quantized) Neural
  Networks.
\newblock \emph{arXiv preprint arXiv:2106.12912}.

\bibitem[{Maier et~al.(2019)Maier, Syben, Stimpel, W{\"u}rfl, Hoffmann,
  Schebesch, Fu, Mill, Kling, and Christiansen}]{mai19}
Maier, A.~K.; Syben, C.; Stimpel, B.; W{\"u}rfl, T.; Hoffmann, M.; Schebesch,
  F.; Fu, W.; Mill, L.; Kling, L.; and Christiansen, S. 2019.
\newblock {Learning with known operators reduces maximum error bounds}.
\newblock \emph{Nature machine intelligence}, 1(8): 373--380.

\bibitem[{Meske et~al.(2022)Meske, Bunde, Schneider, and Gersch}]{mesk22}
Meske, C.; Bunde, E.; Schneider, J.; and Gersch, M. 2022.
\newblock Explainable artificial intelligence: objectives, stakeholders, and
  future research opportunities.
\newblock \emph{Information Systems Management}, 39(1): 53--63.

\bibitem[{Nayman et~al.(2022)Nayman, Golbert, Noy, Ping, and
  Zelnik-Manor}]{nay22}
Nayman, N.; Golbert, A.; Noy, A.; Ping, T.; and Zelnik-Manor, L. 2022.
\newblock {Diverse Imagenet Models Transfer Better}.
\newblock \emph{arXiv preprint arXiv:2204.09134}.

\bibitem[{Poggio, Banburski, and Liao(2020)}]{pog20}
Poggio, T.; Banburski, A.; and Liao, Q. 2020.
\newblock Theoretical issues in deep networks.
\newblock \emph{Proceedings of the National Academy of Sciences}, 117(48):
  30039--30045.

\bibitem[{Saxe et~al.(2019)Saxe, Bansal, Dapello, Advani, Kolchinsky, Tracey,
  and Cox}]{saxe19}
Saxe, A.~M.; Bansal, Y.; Dapello, J.; Advani, M.; Kolchinsky, A.; Tracey,
  B.~D.; and Cox, D.~D. 2019.
\newblock On the information bottleneck theory of deep learning.
\newblock \emph{Journal of Statistical Mechanics: Theory and Experiment},
  2019(12): 124020.

\bibitem[{Schneider(2022)}]{sch22}
Schneider, J. 2022.
\newblock Correlated Initialization for Correlated Data.
\newblock \emph{Neural Processing Letters}, 1--18.

\bibitem[{Schneider and Vlachos(2021)}]{sch21cla}
Schneider, J.; and Vlachos, M. 2021.
\newblock Explaining neural networks by decoding layer activations.
\newblock In \emph{International Symposium on Intelligent Data Analysis},
  63--75.

\bibitem[{Schneider and Vlachos(2022)}]{sch22exp}
Schneider, J.; and Vlachos, M. 2022.
\newblock {Explaining classifiers by constructing familiar concepts}.
\newblock \emph{Machine Learning}, 1--34.

\bibitem[{Shwartz-Ziv and Tishby(2017)}]{tish17}
Shwartz-Ziv, R.; and Tishby, N. 2017.
\newblock Opening the black box of deep neural networks via information.
\newblock \emph{arXiv preprint arXiv:1703.00810}.

\bibitem[{Simonyan and Zisserman(2014)}]{sim14}
Simonyan, K.; and Zisserman, A. 2014.
\newblock Very deep convolutional networks for large-scale image recognition.
\newblock \emph{Int. Conference on Learning Representations (ICLR)}.

\bibitem[{Tapia and Est{\'e}vez(2020)}]{tap20}
Tapia, N.~I.; and Est{\'e}vez, P.~A. 2020.
\newblock On the information plane of autoencoders.
\newblock In \emph{2020 International Joint Conference on Neural Networks
  (IJCNN)}, 1--8.

\bibitem[{Tishby, Pereira, and Bialek(2000)}]{tish00}
Tishby, N.; Pereira, F.~C.; and Bialek, W. 2000.
\newblock The information bottleneck method.
\newblock \emph{arXiv preprint physics/0004057}.

\bibitem[{Tishby and Zaslavsky(2015)}]{tish15}
Tishby, N.; and Zaslavsky, N. 2015.
\newblock Deep learning and the information bottleneck principle.
\newblock In \emph{IEEE Information Theory Workshop (ITW)}, 1--5.

\bibitem[{Wang et~al.(2021)Wang, Huang, Kuruoglu, Sun, Chen, and
  Zheng}]{wang21pac}
Wang, Z.; Huang, S.-L.; Kuruoglu, E.~E.; Sun, J.; Chen, X.; and Zheng, Y. 2021.
\newblock PAC-Bayes Information Bottleneck.
\newblock \emph{arXiv preprint arXiv:2109.14509}.

\bibitem[{Xiao, Rasul, and Vollgraf(2017)}]{xia17}
Xiao, H.; Rasul, K.; and Vollgraf, R. 2017.
\newblock Fashion-mnist: a novel image dataset for benchmarking machine
  learning algorithms.
\newblock \emph{arXiv preprint arXiv:1708.07747}.

\bibitem[{Zeiler and Fergus(2014)}]{zeil14}
Zeiler, M.~D.; and Fergus, R. 2014.
\newblock Visualizing and understanding convolutional networks.
\newblock In \emph{European conference on computer vision}.

\end{thebibliography}

\ifapp
\section{Appendix}
We provide theoretical analysis first and then also provide more empirical analysis.

\subsection{Sums of weights and symmetries} \label{sec:sym}

Due to Def. \ref{def:dat} we have $x_i=v$ for $x\in C^y_v$ with $v \in \{0,k\}$ for $i \in A_{y}$ and $x_i=0$ otherwise, yielding using Eq \ref{eq:oy}:
\eq{
o_{y'}(C^y_v)= \sum_{i<d} w_{y',i}\cdot x_i=v\cdot \sum_{i \in A_y} w_{y',i} = v\cdot S_{y,y'}  \label{eq:oy}
}
where we defined the sum of weights $S_{y,y'}$ for input attributes $A_y$ leading to outputs of class $y'$ as
\eq{
S_{y,y'}:=\sum_{i \in A_{y}} w_{y',i} \label{eq:sums0}
}
Weights $w_{y',i}=w_{y',j}$ are identical for $i,j \in A_y$:
Initially, all $w^{(0)}_{y',i}$ are equal, i.e. $1/l$ (Def \ref{ass:initw}). The update is the same for all $j \in A_y$ due to Eq. \ref{eq:dLy}. Therefore, $w^{(t)}_{y',i}=w^{(t)}_{y',j}$ for $i,j \in A_y$. 
Thus, for any $t$ holds:
\eq{
S^{(t)}_{y,y'}=\sum_{i \in A_{y}} w^{(t)}_{y',i}=|A_{y}| w^{(t)}_{y',i \in A_y}=\frac{d}{l} w^{(t)}_{y',i \in A_y} \label{eq:sums}
}
We also use the change $\partial S^{(t)}_{y,y'}$ of $S^{(t)}_{y,y'}$ in iteration $t$, i.e.:
\eq{
&S^{(t+1)}_{y,y'}=S^{(t)}_{y,y'}-\partial S^{(t)}_{y,y'}\\
&\partial S^{(t)}_{y,y'}:= \frac{\lambda}{|D|} \frac{d}{l} \sum_{x,y \in D| i \in A_y} \frac{\partial L(x,y)}{\partial w_{y',i \in A_y}}\\
&=\frac{\lambda}{n} \frac{d}{l} \frac{n}{2l} \sum_{v \in \{1,k\}}\Bigl( v\cdot(q^{(t)}(y'|C^y_v)-\vmathbb{1}_{y=y'})\Bigr) \\ 
&\teb{ using Eq. \ref{eq:dLy} and symmetries (see \ifapp Eq. \ref{eq:dlsum}  in \fi extended version)}\\
&=\frac{\lambda d}{2l^2}\cdot \sum_{v \in \{1,k\}}\Bigl( v\cdot(q^{(t)}(y'|C^y_v)-\vmathbb{1}_{y=y'})\Bigr) \label{eq:dSy} 
}
Sums of weights are identical across classes $S_{y,y'}=S_{y'',y'''}$ given that either ($y=y'$ and $y''=y'''$) or ($y\neq y'$ and $y''\neq y'''$). This is easy to see, since, initially, all weights are equal (Def \ref{ass:initw}) and the samples of each class are equal aside from symmetry. 


Formally, we aim to show that sums of weights are identical across classes $S_{y,y'}=S_{y'',y'''}$ given that either ($y=y'$ and $y''=y'''$) or ($y\neq y'$ and $y''\neq y'''$). On a high level this follows since, initially, all weights are equal and the samples of each class are equal aside from symmetry.

Furthermore, not only are weights within $j \in A_y$ are identical, but also sum of weights of different classes for an iteration. By definition samples of all classes are symmetric, classes are balanced and also all weights are identically initialized. Therefore, the change of a weight $w_{y',j}$, when we sum across all samples $(x,y) \in D$ becomes based on Eq. \ref{eq:dLy}  for a fixed $j$:
\eq{
&\sum_{x,y \in D}  \frac{\partial L(x,y)}{\partial w_{y',j}}\\
&=\sum_{x,y \in C^y|j \in A_y} \frac{\partial L(x,y)}{\partial w_{y',j}}+\sum_{x,y \in C^y|j \nin A_y} \frac{\partial L(x,y)}{\partial w_{y',j}}\\
&\teb{\phantom{a}}+\sum_{x,y \in D\setminus C^y} \frac{\partial L(x,y)}{\partial w_{y',j}}\\
&=\sum_{x,y \in C^y|j \in A_y} \frac{\partial L(x,y)}{\partial w_{y',j}} \teb{other terms are 0 (Eq \ref{eq:dLy})}\\
&=\sum_{v \in \{1,k\}} \sum_{x \in C^{y}_v|j \in A_y} \frac{\partial L(x,y)}{\partial w_{y',j}}\\
&=\sum_{v \in \{1,k\}} \sum_{x \in C^{y}_v|j \in A_y} \frac{\partial L(x,y)}{\partial w_{y',j}}\\
&=\sum_{v \in \{1,k\}} \sum_{x \in C^{y}_v|j \in A_y}  \frac{\partial L(x,y)}{do_{y'}} \cdot \frac{do_{y'}}{\partial w_{{y'},j}} \\
&=\sum_{v \in \{1,k\}} \sum_{x \in C^{y}_v|j \in A_y}  (q_{y'|x}-\vmathbb{1}_{y'=y})\cdot x_j\\
&=\sum_{v \in \{1,k\}} \sum_{x \in C^{y}_v}  (q_{y'|C^{y}_v}-\vmathbb{1}_{y'=y})\cdot v\\
&=\sum_{v \in \{1,k\}} |C^{y'}_v|  (q_{y'|C^{y}_v}-\vmathbb{1}_{y'=y})\cdot v\\
&=\sum_{v \in \{1,k\}} |\frac{n}{2l}|  (q_{y'|C^{y}_v}-\vmathbb{1}_{y'=y})\cdot v \label{eq:dlsum}
}

Thus, we need to show that for every two pair of classes $(y,y')$ and $(y'',y''')$ with (a) ($y=y'$ and $y''=y'''$) and (b) ($y\neq y'$ and $y''\neq y'''$), we have $w_{y',j \in A_y}=w_{y''',j \in A_{y''}}$ and $q(y'|C^{y}_v)=q(y'''|C^{y''}_v)$.

The proof is by induction. For the base case of the induction, i.e., the initial condition $t=0$, this follows since all weights are initialized to the same value $1/d$ and using Eq \ref{eq:soft}, i.e., for Softmax we have:
\eq{
q(y'|C^{y}_v)&=\frac{e^{o_{y'(C^y_v)}}}{\sum_{y''<l}e^{o_{y''}(C^y_v)}} \teb{ using Eq. \ref{def:q} } \\
&=\frac{e^{v\cdot S_{y,y'}}}{\sum_{y''<l}e^{v\cdot S_{y,y''}}} \teb{ using Eq. \ref{eq:oy}} \\
&=\frac{e^{v\cdot \frac{d}{l} w_{y',i \in A_y}}}{\sum_{y''<l}e^{v\frac{d}{l} w_{y'',i \in A_y}}} \teb{ using Eq. \ref{eq:sums} }  \label{eq:soft}\\
}

For the induction step from $t$ to $t+1$, we consider the update of a weight $w_{y',i}$ based on Eq. \ref{eq:grad} and the assumption  which holds for $t=0$ (Eq. \ref{eq:oy}), i.e.,
\eq{&
q^{(t)}(y|C^y_v)=q^{(t)}(y''|C^{y''}_v) \label{eq:symq}\\
& q^{(t)}(y'|C^y_v)=q^{(t)}(y'''|C^{y''}_v) \teb{ for $y\neq y'$ and $y''\neq y'''$} 
} 
Thus, for the gradient we have using Eq \ref{eq:dlsum}:
\eq{
&\sum_{v \in \{1,k\}} \sum_{x \in C^{y}_v} \frac{\partial L(x,y)}{\partial w_{y',j \in A_y}}\\
&=\sum_{v \in \{1,k\}} |\frac{n}{2l}|  (q_{y'|C^{y}_v}-\vmathbb{1}_{y'=y})\cdot v\\
&=\sum_{v \in \{1,k\}} |\frac{n}{2l}|  (q_{y'''|C^{y''}_v}-\vmathbb{1}_{y'''=y''})\cdot v \\
& \teb{ Using Induction assumption Eq \ref{eq:symq}}\\
&=\sum_{v \in \{1,k\}} \sum_{x \in C^{y'''}_v} \frac{\partial L(x,y''')}{\partial w_{y'',j \in A_{y'''}}}
}
Using Equation \ref{eq:grad} we get:
\eq{
w^{(t+1)}_{y,j}&=w^{(t)}_{y,j}-\frac{\lambda}{|D|} \sum_{(x,y') \in D}  \nabla_{w_{y,j}} L(x,y') \\
&=w^{(t)}_{y,j}-\frac{\lambda}{|D|} \sum_{v \in \{1,k\}} \sum_{x \in C^{y}_v} \frac{\partial L(x,y)}{\partial w_{y',j \in A_y}} \\
&=w^{(t)}_{y'',j}-\frac{\lambda}{|D|} \sum_{v \in \{1,k\}} \sum_{x \in C^{y}_v} \frac{\partial L(x,y)}{\partial w_{y'',j \in A_y}} \\
&=w^{(t+1)}_{y'',j}
}

\subsection{Reconstruction error for fitted points} \label{sec:zerorec}
For reconstructing attribute $j$, we fit a line $g_j$ without error through two points $(q(y|C^{y'\neq y}_k),0)$ and $(q(y|C^y_k),k)$ as illustrated in Figure \ref{fig:linvis}. Thus, it follows from the definition of $g_j$ that the reconstruction errors $R_k$ (i.e., for $x_j=k$ and $x \in  C^y_k$) and  $R_{0,C^{y'}_k}$ (i.e., for $x_j=0$ and $x \in C^{y'\neq y}_k)$ are both zero. We show this formally below.

The reconstruction error for $x_j=k$, i.e., $R_{k}$  is
\eq{
&R_{k}=\frac{1}{l}\cdot\bigl(k-  \frac{k}{q(y|C^y_k)-q(y|C^{y'\neq y}_k}) \cdot (q(y|C^y_k)-q(y|C^{y'\neq y}_k))\bigr)^2\\
&=\frac{1}{l}\cdot(k-  k)^2=0
}

Next, we compute the reconstruction error $R_{0,C^{y'}_k}$ for $x_i=0$. We reconstruct $x_i=0$ perfectly for $x \in C^{y'\neq y}_k$:
\eq{
&R_{0,C^{y'}_k}=(1-2/l)\cdot\bigl(0-  \frac{k}{q(y|C^{y'\neq y}_k)-q(y|C^{y'\neq y}_k)} \\
&\teb{ } \cdot (q(y|C^{y'\neq y}_k)-q(y|C^{y'\neq y}_k))\bigr)^2\\
&=(1-2/l)\cdot\bigl(0-  \frac{k}{q(y|C^y_k)-q(y|C^{y'\neq y}_k} \cdot 0\bigr)^2=0
}

\subsection{Series expansions} \label{sec:series}
\noindent \textbf{Prerequisites}: As a prerequisite, we state series expansions commonly used in some of our proofs. These expansions are crucial to simplify expressions with exponential functions arising from the softmax function. We only use first-order expansions, which suffice to obtain reconstruction errors of orders $O(1/l)$. Using higher orders would allow us to get an approximation error of order $O(1/l^j)$ for some $j>1$ at the price of more complex expressions. The Taylor series of the exponential function is:

\eq{
&e^{x}=\sum_{i=0}^{\infty} x^i/i!=1+x+x^2/2+x^3/6+O(x^4/(4!)) \label{eq:ser}
}
We derive series based on the assumption that $f$ and $u$ are of the same order, but both are small, i.e., $f \in O(1/l)$ and $u \in O(1/l)$, and $k \in O(1)$. So that we can neglect higher order terms in both $f$ and $u$. 
\eq{
&e^{fk}-e^{ku}=1+fk +O((fk)^2\\
&\teb{ }-(1+ku+O((ku)^2)\\
&= k(f-u) +O_0 \label{eq:efkeku}\\
& \teb{ with } O_0:=O(k^2(f-u)^2) \label{eq:O0}
}

\eq{
&e^{f+ku}- e^{(k+1)u}=1+(f+ku)+O((f+ku)^2)\\
&\teb{\phantom{ab} }-(1+(k+1)u+O(((k+1)u)^2))\\
&=(f-u)+O_1 \label{eq:efkueku}\\
& \teb{ with } O_1:=O((f+ku)^2-((k+1)u)^2)\\
& \teb{ \phantom{abcd  } }=O((f-u)(f+2ku+u)) \label{eq:O1}
}
\eq{
&e^{fk+u}-e^{f+ku}\\
&=1+(fk+u)+O((fk+u)^2)\\
&-(1+(f+ku)+O((f+ku)^2)) \\
&=(f(k-1)-(k-1)u)\\
&+ O((fk+u)^2)-O((f+ku)^2)) \\
&=(k-1)(f-u)+O_2 \label{eq:efkuefku}\\
& \teb{ with } O_2:=O((k^2-1)(f-u)(f+u))  \label{eq:O2}
}
We also derive series expansions for the same three terms under the premise that $u \in O(f/l)$, i.e., $u$ is much smaller than $f$, and $f\in[1/l,1]$. This essentially allows us to discard terms with $u$ to obtain an error bound of $O(f/l)$. The results follow from direct application of Eq \ref{eq:ser}.
\eq{
&e^{fk}-e^{ku}=e^{fk}-1+O(f/l) \label{eq:fkku2}\\
&e^{f+ku}- e^{(k+1)u}=e^{ku}(e^{f}- e^{u})\\
&= (1+ O(f/l)) (e^{f}- 1+O(f/l))\\
&= e^{f}- 1+O(f/l) \label{eq:fkuku2}\\
&e^{fk+u}-e^{f+ku}=e^u\cdot e^{fk}-e^{ku}\cdot e^{f} \\
&=(1+O(f/l)\cdot e^{fk}-(1+O(f/l)\cdot e^{f} \\
&=e^{fk}-e^{f}+O(f/l)\label{eq:fkufku}
}
\eq{
&\frac{1}{e^f+le^{fu}} \\
&=\frac{1}{O(1)+le^{fu}}\\
&=\frac{1}{O(1)+l(1+O(1/l))}\\
&=\frac{1}{l+O(1)}\label{eq:fraffu}
}

Finally, we also as an expansion that holds as long as $f\in[1/l,1]$.
\eq{
&e^{fu}=1+O(f/l) \label{eq:efu}\\
}

\subsection{Main Theorems with Proofs} \label{sec:qbproofs}
We replicate the theorems from the main text (but keep the theorem counter going).

Symmetry allows us to focus on a class $y$ without loss of generality.  It is also easy to see that all output probabilities $q$ are the same for all classes and samples initially.

\begin{thm} \label{eq:qinit}
Initially, $q^{(0)}(y'|x)=1/l$ for any sample $(x,y) \in D$ and any class $y'$.
\end{thm} 
\begin{proof}
\eq{
&q^{(0)}(y''|C^y_v)=\frac{e^{o^{(0)}_{y''}}}{\sum_{y'<l}e^{o^{(0)}_{y'}}}  \teb{using Eq. \ref{def:q}}\\
&= \frac{e^{v\cdot S^{(0)}_{y,y''}}}{\sum_{y'<l}e^{v\cdot S^{(0)}_{y,y'}}} \teb{using Eq. \ref{eq:oy} }\\
&= \frac{e^{v/l}}{l\cdot e^{v/l}} =\frac{1}{l} \\
&\teb{using $S^{(0)}_{y,y'''}=d/l\cdot 1/d=1/l$ (Eq. \ref{eq:sums} and $w^{(0)}_{i,j}=1/d$  Ass.\ref{ass:initw})}
}
\end{proof}

After these preparations, we are ready to state the first theorem relating the magnitude of weights, i.e. sums of weights, and iterations $t$.

\begin{thm} \label{thm:boundS}
It holds that $S^{(t-1)}_{y,y}\leq f$ and $S^{(t)}_{y,y}\geq f$  for $f\in[1/l,\log(l)]$ and $S^{(t)}_{y,y'\neq y} = \frac{1-O(f-1/l)}{l}$ for $t \in  \Theta(\frac{2l^2(f-1/l)}{\lambda d(k+1)})$.
\end{thm} 

\begin{proof} 
Based on the definition of an asymptotically tight bound $\Theta$, we need to show the existence of constants $c_0$ and $c_1$ so that the claims about $S^{(t)}_{y,y}$ and $S^{(t)}_{y,y'\neq y}$ hold for some $t\in [c_0 t^*,c_1 t^*]$ with $t^*=\frac{2l^2(f-1/l)}{\lambda d(k+1)}$. To this end, we first lower and upper bound the change $dS^{(t)}_{y,y}$,  $dS^{(t)}_{y,y'\neq y}$ and the softmax outputs $q^{(t)}$ for  $t\in [0,c_1 t^*]$. Then, we show that the bounds on $S^{(t)}_{y,y}$ follow.

Initial all $q^{(0)}=1/l$ (Thm. \ref{eq:qinit}).

We assume (and later show) that $q^{(t)}(y|C^y_1) \in [0,2/3]$ for $t\leq c_1 t^*$.  

To lower bound the number of iterations, i.e., obtain $c_0 t^*$, we upper bound the change $dS^{(t)}_{y,y}$ (see Eq. \ref{eq:dSy}) and lower bound $dS^{(t)}_{y,y'\neq y}$ for any $t\leq c_1 t^*$. We have the following bounds: 
\eq{ 
dS^{(t)}_{y,y}&>-z\\
dS^{(t)}_{y,y'\neq y}&< z/l \label{eq:lowdsy}\\
&\teb{ with } z:=\frac{\lambda d(k+1)}{2l^2} 
}

The bounds follow by plugging into Eq. \ref{eq:dSy} the values for $q^{(t)}(y'|C^y_v)$ for $v\in\{1,k\}$ with $y=y'$ and $y\neq y'$.
Consider $y\neq y'$: Since $q^{(t)}(y'\neq y|C^y_v)$ decreases with $t$ and, in turn, $dS_{y,y'\neq y}^{(t)}$, we use the initial value to bound $q^{(t)}(y'\neq y|C^y_v)\leq q^{(0)}(y'\neq y|C^y_v)=1/l$ (Eq. \ref{eq:qinit}).
Consider $y= y'$: Since $q^{(t)}(y|C^y_v)$ increases from $t=0$ with growing $t$, which decreases  $dS^{(t)}_{y,y}$, we lower bound $q^{(c_0t^*)}(C^y_v)\geq 0$ (i.e., by definition the output of the softmax cannot be less than 0).

We bound the number of iterations to change $S^{(0)}_{y,y}$ by $S^{(c_0t^*)}_{y,y}-S^{(0)}_{y,y}\leq f -1/l$. Note, $S^{(t)}_{y,y}$ increases over time, thus, $f\geq 1/l$.

The number of iterations $c_0t^*$ is lower bounded by
\eq{
&\frac{f -1/l}{z}=\frac{f -1/l}{\lambda\frac{d(k+1)}{2l^2}}=\\ &\frac{2l^2(f-1/l)}{\lambda d(k+1)}\\
&=2\cdot t^* \teb{ that is } c_0:=2 \label{eq:anac0}
}

\bigskip

To upper bound the number of iterations, i.e., obtain $c_1 t^*$, we proceed analogously.
Since $q^{(0)}(y|C^y_v)$ increases, which decreases $dS^{(t)}_{y,y}$, we use $dS^{(c_1t^*)}_{y,y}$ to lower bound, and in turn, $q^{(c_1t^*)}(y|C^y_1)\leq 2/3$ and $q^{(c_1t^*)}(y|C^y_k)=1$, i.e., that samples $(x,y) \in C^y_k$ are classified without loss. Thus, only samples in $C^y_1$ are assumed to contribute to the change. To lower bound $q^{(c_1t^*)}(y'\neq y | C^y_v)$ we use 0, i.e. there is no change for $dS_{y,y'\neq y}$. 
\eq{ 
 &dS_{y,y}&<-z\cdot 1/6 \\
 &dS_{y,y'\neq y}&\geq 0 
}
This yields analogously as for $c_0$ (Eq. \ref{eq:anac0}) $c_1:=12$.
Thus, we have determined constants $c_0,c_1$ for $t^*$ for bounding $S^{(t)}_{y,y}$. Next, we show that $S^{(t)}_{y,y'\neq y} =1/l-O(\frac{f-1/l}{l})$. To bound the maximal change of $S^{(t)}_{y,y'\neq y}$ we assume $c_1t^*$ iterations and an upper bound on the maximal change $dS_{y'\neq y}<z/l$ (Eq. \ref{eq:lowdsy}).
\eq{
&S^{(t)}_{y,y'\neq y}\geq S^{(0)}_{y,y'\neq y} -c_1t^*\cdot dS_{y,y'\neq y}\\
&= \frac{1}{l}-c_1\cdot \frac{2l^2(f-1/l)}{\lambda d(k+1)} \cdot \frac{\lambda d(k+1)}{2l^3}\\
&= \frac{1}{l}-c_1\cdot (f-1/l) \cdot \frac{1}{l}\\
&= \frac{1}{l}\cdot(1-O(f-1/l))
}
We also have that 
\eq{
&S^{(t)}_{y,y'\neq y}\leq S^{(0)}_{y,y'\neq y}=\frac{1}{l}
}

\bigskip

It remains to show that $q^{(c_1t^*)}(y| C^y_1)\leq 2/3$ for $(x,y)\in C^y_1$.
Using  Eq. \ref{def:q}:
\eq{
q^{(t)}(y|C^y_1)(x)&=\frac{e^{o^{(t)}_{y}}}{\sum_{y'<l}e^{o^{(t)}_{y'}}}\\
&=\frac{e^{S^{(t)}_{y,y}}}{e^{S^{(t)}_{y,y}} +\sum_{y'<l, y'\neq y}e^{S^{(t)}_{y,y'}}}\\
&=\frac{e^f}{e^f+ \sum_{y'<l, y'\neq y}e^{0}}
\teb{ using } S_{y,y}= f \\
&\text{ and the lower bound for }S_{y,y'\neq y}\geq 0\\ 
&=\frac{e^f}{e^f+ l-1} \\
&\leq \frac{l}{2l-1}  \teb{ using } f\leq \log(l)\\
&\leq \frac{2}{3}  \teb{ using } l\geq 2
}
\end{proof}

In multiple proofs we use Theorem \ref{thm:boundS} and use $u$ as abbreviation, i.e.:
\eq{
&u:=S^{(t)}_{y,y'\neq y} =1/l-O(\frac{f-1/l}{l})=\frac{1-c(f-1/l)}{l} \label{def:u} 
}
where $c$ is some constant from the $O$-notation. The bounds for the reconstruction errors for the stages are:




Finally, we look at the boundaries, i.e., $t=0$ and $t\rightarrow \infty$.

\begin{thm}\label{le:recBInit}  
The reconstruction error of $g_j$ for $t=0$ is $R^{(0)}=\Omega(k^2/l)$ and for $t\rightarrow \infty$ it is $R^{(t\rightarrow\infty)}=\Omega(k/l)$. The errors $R^{(0)}$ and $R^{(t\rightarrow\infty)}$ for $g_j$ are asymptotically optimal.
\end{thm}

\subsection{Proofs for reconstruction errors $R_1$ and $R_{0,C^{y'\neq y}_1}=O(k^2/l^2)$}

\begin{thm} \label{thm:R1}
The reconstruction error $R_1=O(k^4/l^2)$ for $1/l< f<c_f/l$ and $R_1=\Omega(k)$ for $c_f/l\leq f<1$  for an arbitrary constant $c_f>1$.
\end{thm}

\begin{proof}

The denominator of the slope in the reconstruction function $g_j$ (Eq. \ref{eq:gj}) is:
\eq{
&q(y|C^y_k)-q(y'\neq y|C^y_k) \label{eq:cyk}\\
&=\frac{e^{fk}}{e^{fk}+e^{ku}\cdot (l-1)} -\frac{e^{ku}}{e^{fk}+e^{ku}\cdot (l-1)} \\
&=\frac{e^{fk} -e^{ku}}{e^{fk}+e^{ku}\cdot (l-1)} 
}

First, we reconstruct $x_i=1$ using $q(y|C^y_1)$ and subtract the ``shift'' $q(y'\neq y|C^y_k)$ applied in $g_j$ (Eq. \ref{eq:gj}).
\eq{
&q(y|C^y_1)-q(y'\neq y|C^y_k) \label{eq:cy1}\\
&=\frac{e^{f}}{e^{f}+e^{u}\cdot (l-1)} -\frac{e^{ku}}{e^{fk}+e^{ku}\cdot (l-1)} \\
&=\frac{e^{f}(e^{fk}+e^{ku}\cdot (l-1))-e^{ku}(e^{f}+e^{u}\cdot (l-1))}{(e^{f}+e^{u}\cdot (l-1)) \cdot (e^{fk}+e^{ku}\cdot (l-1))} \\
&=\frac{e^{f(k+1)}+e^{f+ku}\cdot (l-1)-e^{ku+f}- e^{(k+1)u}\cdot (l-1)}{(e^{f}+e^{fu}\cdot (l-1)) \cdot (e^{fk}+e^{ku}\cdot (l-1))}\\
&=\frac{e^{f(k+1)}+e^{f+ku}\cdot (l-2)- e^{(k+1)u}\cdot (l-1)}{(e^{f}+e^{fu}\cdot (l-1)) \cdot (e^{fk}+e^{ku}\cdot (l-1))}
}

We compute the ratio of Eq. (\ref{eq:cyk}) and  (\ref{eq:cy1}). 
\eq{
&\frac{q(y|C^y_1)-q(y'\neq y|C^y_k)}{q(y|C^y_k)-q(y'\neq y|C^y_k)} \label{eq:rat1}\\
&=\frac{e^{f(k+1)}+e^{f+ku}\cdot (l-2)- e^{(k+1)u}\cdot (l-1)}{(e^{f}+e^{fu}\cdot (l-1)) \cdot (e^{fk}+e^{ku}\cdot (l-1))}\cdot\frac{e^{fk}+e^{ku}\cdot (l-1)}{e^{fk} -e^{ku}}\\
&=\frac{e^{f(k+1)}+e^{f+ku}\cdot (l-2)- e^{(k+1)u}\cdot (l-1)}{(e^{f}+e^{fu}\cdot (l-1)) \cdot (e^{fk} -e^{ku})}\\
&=\frac{e^{f(k+1)}+l(e^{f+ku}- e^{(k+1)u})+O(1)}{(e^{f}+e^{fu}\cdot (l-1)) \cdot (e^{fk} -e^{ku})} \label{eq:rat1fi}
}

We derive approximations for the ratio (Eq. \ref{eq:rat1}) distinguishing  different ranges of values  $f$. 
\eq{
&\text{Assume: } 1/l\leq f<1 \label{eq:rat1sf}\\
&\frac{q(y|C^y_1)-q(y'\neq y|C^y_k)}{q(y|C^y_k)-q(y'\neq y|C^y_k)}\\
&=\frac{e^{f(k+1)}+l(e^{f+ku}- e^{(k+1)u})+O(1)}{(e^{f}+e^{fu}\cdot (l-1)) \cdot (e^{fk} -e^{ku})} \teb{ see } Eq. (\ref{eq:rat1fi})\\
&=\frac{l(e^{f+ku}- e^{(k+1)u})}{(l+O(1)) \cdot (e^{fk} -e^{ku})}+O(1/l)\\
&\teb{using Eq. \ref{eq:fraffu}}\\
&=\frac{e^{f+ku}- e^{(k+1)u}}{(e^{fk} -e^{ku})}+O(1/l) \label{eq:lf}
}

Consider an arbitrary constant $c_f>1$.

\eq{
&\text{Assume: } 1/l\leq f<c_f/l \label{eq:rat1sfc}\\
&\frac{q(y|C^y_1)-q(y'\neq y|C^y_k)}{q(y|C^y_k)-q(y'\neq y|C^y_k)}\\
&=\frac{(f-u)+O_1}{k(f-u) +O_0}+O(1/l) \\
&\teb{ using } Eq (\ref{eq:lf}), (\ref{eq:efkeku})\text{ and }(\ref{eq:efkueku}) \nonumber\\
&=\frac{1+O_1/(f-u))}{k+O_0/(k(f-u))}+O(1/l)\\
&=\frac{1}{k}+O(k/l)\\
&\teb{using Eq. \ref{eq:O1} and \ref{eq:O0} giving: }  \nonumber \\
& O_1/(f-u)=O(k/l) \text { and } O_0/(k(f-u))= O(k/l) \nonumber
}

Next, we compute the reconstruction error $R_{1}$ based on Assumption \ref{eq:rat1sfc}.
\eq{
&R_{1}=(1-k(\frac{1}{k}+O(k/l)) )^2 \label{eq:reck0}\\
&=(O(k^2/l))^2=O(k^4/l^2)
}

\eq{
&\text{Assume: } c_f/l\leq f<1 \label{eq:rat1sfc3}\\
&\frac{q(y|C^y_1)-q(y'\neq y|C^y_k)}{q(y|C^y_k)-q(y'\neq y|C^y_k)}\\
&=\frac{e^{f+ku}- e^{(k+1)u}}{(e^{fk} -e^{ku})}+O(1/l) \teb{ using } Eq (\ref{eq:lf})\\
&=\frac{e^{f}- 1}{e^{fk} -1}+O(f/l) \teb{ using } Eq (\ref{eq:fkuku2}) and (\ref{eq:fkku2})\\
&=\frac{\sum_{i=1}^{\infty} f^i/i!}{\sum_{i=1}^{\infty} (fk)^i/i!}+O(f/l) \teb{ using } Eq. (\ref{eq:ser})\\
&=\frac{\sum_{i=1}^{\infty} 1/i!}{\sum_{i=1}^{\infty} k^i/i!}+O(f/l)\\
&=\frac{e-1}{e^k-1}+O(f/l)\\
}
Next, we compute the reconstruction error $R_{1}$ based on Assumption \ref{eq:rat1sfc3}.
\eq{
&R_{1}=(1-k(\frac{e-1}{e^k-1})-O(kf/l) )^2 \label{eq:reck1}\\
&=\Omega(k)\\
&\teb{ since $x(e-1)/(e^x-1)<1$ for $x>1$}
}

\eq{
&\text{Assume: } 1\leq f \label{eq:rat1sfc1}\\
&\frac{q(y|C^y_1)-q(y'\neq y|C^y_k)}{q(y|C^y_k)-q(y'\neq y|C^y_k)}\\
&=\frac{e^{f(k+1)}+l(e^{f+ku}- e^{(k+1)u})+O(1)}{(e^{f}+e^{fu}\cdot (l-1)) \cdot (e^{fk} -e^{ku})}\\
&=\frac{e^{f(k+1)}+l(e^{f+ku}- e^{(k+1)u})}{(e^{f}+e^{fu}\cdot l) \cdot (e^{fk} -e^{ku})} +O(1/(le^{fk}))\\
&=\frac{e^{f(k+1)}+l(e^{f+ku}- e^{(k+1)u})}{(e^{f}+l) \cdot (e^{fk} -e^{ku})} +O(1/(le^{fk}))\\
&=\frac{e^{f(k+1)}+l(e^{f}- 1))}{(e^{f}+l) \cdot (e^{fk} -1)} +O(1/(le^{fk}))\\
&=\frac{l(e^{f}- 1))}{(e^{f}+l) \cdot (e^{fk} -1)}+O(1/(le^{fk})) \\
&+\frac{e^{f(k+1)}}{(e^{f}+l) \cdot (e^{fk} -1)} \\
&=\frac{l(e^{f}- 1))}{(e^{f}+l) \cdot (e^{fk} -1)}+O(1/l) \teb{ using } e^{fk}<l/2\\
}
The reconstruction error $R_1$ becomes based on Ass. \ref{eq:rat1sfc1}.
\eq{
&R_{1}=(1-k\cdot (\frac{l(e^{f}- 1)}{(e^{f}+l) \cdot (e^{fk} -1)}+O(1/l)))^2\\
&R_{1}=(1-k\cdot (c\cdot \frac{le^{f}}{le^{fk}}+O(1/l)))^2 \\
&\teb{for a constant } 1/2<c<2\\
&R_{1}=(1- \frac{kc}{e^{f(k-1)}}-O(k/l))^2 = O(1) \label{eq:recf2} 
}

\end{proof}

\begin{thm}\label{thm:R0}
The reconstruction error $R_{0,C^{y'\neq y}_1}=O(k^2/l^2)$ for $1/l< f$ 
\end{thm}

\begin{proof}

We first bound the difference between $q(y'\neq y|C^y_1)-q(y'\neq y|C^y_k)$.
\eq{
&q(y'\neq y|C^y_1)-q(y'\neq y|C^y_k) \label{eq:cy1n}\\
&=\frac{e^{u}}{e^{f}+e^{u}\cdot (l-1)}-\frac{e^{ku}}{e^{fk}+e^{ku}\cdot (l-1)} \\
&=\frac{e^{u}(e^{fk}+e^{ku}\cdot (l-1))-e^{ku}(e^{f}+e^{u}\cdot (l-1))}{(e^{fk}+e^{ku}\cdot (l-1)) \cdot (e^{f}+e^{u}\cdot (l-1))} \\
&=\frac{e^{u}e^{fk}-e^{ku}e^{f}}{(e^{fk}+e^{ku}\cdot (l-1)) \cdot (e^{f}+e^{u}\cdot (l-1))} \\
&=\frac{e^{fk+u}-e^{f+ku}}{(e^{fk}+e^{ku}\cdot (l-1)) \cdot (e^{f}+e^{u}\cdot (l-1))} 
}

We compute the ratio of Eq. (\ref{eq:cyk}) and  (\ref{eq:cy1n}). 
\eq{
&\frac{q(y'\neq y|C^y_1)-q(y'\neq y|C^y_k)}{q(y|C^y_k)-q(y'\neq y|C^y_k)} \label{eq:rat2}\\
&=\frac{e^{fk+u}-e^{f+ku}}{(e^{fk}+e^{ku}\cdot (l-1)) \cdot (e^{f}+e^{u}\cdot (l-1))} \cdot 
&\phantom{ab}\frac{e^{fk}+e^{ku}\cdot (l-1)}{e^{fk} -e^{ku}}\\
&=\frac{e^{fk+u}-e^{f+ku}}{(e^{f}+e^{u}\cdot (l-1)) \cdot (e^{fk} -e^{ku})}
}

We derive approximations for the ratio (Eq. \ref{eq:rat2}) distinguishing  different ranges of values  $f$.
\eq{
&\frac{q(y'\neq y|C^y_1)-q(y'\neq y|C^y_k)}{q(y|C^y_k)-q(y'\neq y|C^y_k)}\\
&=\frac{e^{fk+u}-e^{f+ku}}{(e^{f}+e^{u}\cdot (l-1)) \cdot (e^{fk} -e^{ku})}  \label{eq:bou2}\\
&=\frac{e^{fk+u}-e^{f+ku}}{(l+O(1)) \cdot (e^{fk} -e^{ku})}+O(1/l)\\
&\teb{using Eq. \ref{eq:fraffu}}
}
\eq{
&\text{Assume: } 1/l\leq f<c_f/l \label{eq:rat2sf}\\
&=\frac{(k-1)(f-u)+O_2 }{(l+O(1)) \cdot (k(f-u) +O_0)} \\
&\teb{ using } Eq.\ref{eq:rat2},\ref{eq:efkeku},\ref{eq:efkuefku}\\
&=\frac{(1-1/k)+O_2/(k(f-u)) }{l(1 +O_0/(k(f-u)))}+O(1/l) \\
&=\frac{1-1/k}{l}+O(k/l) \\
&\teb{using Eq. \ref{eq:O2} and \ref{eq:O0} giving:} \\
&\teb{\phantom{abcd}} O_2/(k(f-u))=O(k/l) \text{ and } O_0/(k(f-u))= O(k/l)
}
Next, we compute the reconstruction error $R_{0,C^{y'\neq y}_1}$ based on Assumption \ref{eq:rat2sf}.
\eq{
&R_{0,C^{y'\neq y}_1}=(0-k(\frac{1-1/k}{l}+O(k/l)^2\\
&=O((k/l)^2)=O(k^2/l^2)\\
}
\eq{
&\text{Assume: } c_f/l\leq f<1 \label{eq:rat2bf}\\
&=\frac{e^{fk+u}-e^{f+ku}}{(l+O(1)) \cdot (e^{fk} -e^{ku})}+O(1/l)\\
&=\frac{e^{fk}-e^{f}}{l(e^{fk} -1)}+O(1/l)\\
&=1/l+\frac{1+e^{f}}{l(e^{fk} -1)}+O(1/l)\\
&=1/l+\frac{2+(e^{f}-1}{l(e^{fk} -1)}+O(1/l)\\
&=1/l(1+\frac{2}{e^{fk} -1}+\frac{e-1}{e^{k} -1}+O(1/l)\\
}
Next, we compute the reconstruction error $R_{0,C^{y'\neq y}_1}$ based on Assumption \ref{eq:rat2bf}.
\eq{
&R_{0,C^{y'\neq y}_1}=(0-kO(1/l))^2=O(k^2/l^2)\\
}

\eq{
&\text{Assume: }  f>1\label{eq:rat21} \\
&=\frac{e^{fk+u}-e^{f+ku}}{(l+O(1)) \cdot (e^{fk} -e^{ku})}+O(1/l)\\
&=1/l(1+\frac{2}{e^{fk} -1}+\frac{e-1}{e^{k} -1}+O(f/l)\\
}
Next, we compute the reconstruction error $R_{0,C^{y'\neq y}_1}$ based on Assumption \ref{eq:rat21}.
\eq{
&R_{0,C^{y'\neq y}_1}=(0-kO(1/l))^2=O(k^2/l^2)\\
}
\end{proof}

\begin{thm} \label{thm:R2} 
The reconstruction error $R_1=\Theta(k^2)$, $R_{0,C^{y'\neq y}_1}=O(k^2/l^2)$ for $\frac{\log(l)-1}{k}<f$. 
\end{thm}

\begin{proof}
We examine the second stage, where we assume that $\frac{\log(l)-1}{k}<f\leq \log(l)-1$ or, put differently, $l/2<e^{fk}$ and $e^f\leq l/2$.

We first compute an estimate for the denominator $q(y|C^y_k)-q(y|C^{y'\neq y}_k)$ of the reconstruction function $g_j$ (see Def \ref{eq:gj}).

\eq{
q(y|C^y_k)&= e^{fk} / (e^{fk}+e^{fk/l}\cdot (l-1)) \\
&\geq l/2 / (l/2+e^{0.5}\cdot l) \\
&\geq l/2 / (l/2+2\cdot l)\\
&\geq 1/5 \label{eq:qefk}
}
The upper bound  for $q(y|C^y_k)$ is 1 by the definition of the softmax function.
Thus, since $q(y|C^{y'\neq y}_k)<1/l$ we have for the nominator in $g_j$:
\eq{
&q(y|C^y_k)-q(y|C^{y'\neq y}_k) = \Theta(1) 
} 

To compute $R_1$, we get for the nominator of $g_j$ (see Def \ref{eq:gj}): 
First, we consider $q(y|C^y_1)$ and subtract the shift applied $q(y'\neq y|C^y_k)$ in $g_j$ (Eq. \ref{eq:gj}).
\eq{
&q(y|C^y_1)-q(y'\neq y|C^y_k) \label{eq:cy1a}\\
&=\frac{e^{f(k+1)}+e^{f+ku}\cdot (l-2)- e^{(k+1)u}\cdot (l-1)}{(e^{f}+e^{fu}\cdot (l-1)) \cdot (e^{fk}+e^{ku}\cdot (l-1))}\\
&=\Theta(\frac{e^{fk}+le^{f}}{(e^{f}+l) \cdot e^{fk}})
}

The reconstruction error $R_{1}$ is
\eq{
&R_{1}=(1-  k\cdot\Theta(\frac{e^{fk}+le^{f}}{(e^{f}+l) \cdot e^{fk}}))^2\\
&\teb{ using } e^{fk}\geq l/2, k\geq 2, e^f\leq \sqrt{l}\\
&=(1-  \Theta(k\cdot\frac{l^{3/2}}{l^2}))^2\\
&=(1-  \Theta(k/\sqrt{l}))^2 \label{eq:reck3}\\
&=O(1)
}

Analogously, we investigate the error $R_{0,C^{y'\neq y}_1}$. The nominator of $g_j$ (see Def \ref{eq:gj}) is

\eq{
&q(y'\neq y|C^y_1)-q(y'\neq y|C^y_k) \label{eq:cy1na}\\
&=\frac{e^{fk+u}-e^{f+ku}}{(e^{fk}+e^{ku}\cdot (l-1)) \cdot (e^{f}+e^{u}\cdot (l-1))} \\
&=\frac{e^{fk}}{(e^{fk}+l) \cdot (e^{f}+l)} \\
&=O(\frac{1}{e^{f}+l})
}

The reconstruction error $R_{0,C^{y'\neq y}_1}$  is
\eq{
R_{0,C^{y'\neq y}_1}&=(0-  kO(\frac{1}{e^{f}+l}))^2\\
&= O(\frac{k}{e^{f}+l}))^2 = O(\frac{k^2}{l^2}) \label{rec:sec2}
}


Next, we examine the third stage, which is said to begin once $e^f>l/2$.    

\eq{
q(y|C^y_1)&= e^{f} / (e^{f}+e^{u}\cdot l) \\
&\geq l/2 / (l/2+e^{0.5}\cdot l) \\
&\geq l/2 / (l/2+2\cdot l)
&\geq 1/5
}
Since $k>1$, we have $q(y|C^y_k)\geq  q(y|C^y_1)\geq 1/5$ and, thus, $q(y|C^y_1)/q(y|C^y_k) = \Theta(1)$

The reconstruction error $R_{1}$  is
\eq{
R_{1}&=(1-  \Theta(k)\cdot \Theta(1))^2\label{eq:reck4} \\
&=\Theta(k^2)
}
The reconstruction error $R_{0,C^{y'\neq y}_1}=O(\frac{k^2}{l^2})$ and can be obtained identically as in Eq. \ref{rec:sec2}.

\end{proof}

\begin{thm}
The reconstruction error of $g_j$ for $t=0$ is $R^{(0)}=\Omega(k^2/l)$ and for $t\rightarrow \infty$ it is $R^{(t\rightarrow\infty)}=\Omega(k/l)$. The errors $R^{(0)}$ and $R^{(t\rightarrow\infty)}$ for $g_j$ are asymptotically optimal.
\end{thm}

\begin{proof}
We show that the error is optimal for $t=0$  and within a factor of 2 of the optimal for $t\rightarrow \infty$.

For $t=0$ the output $q$ is the same for all inputs $x$, i.e., $1/l$ (see Eq. (\ref{eq:qinit}) and  Figure \ref{fig:linvis}). Thus, the constant reconstruction is optimal, i.e., the average of all outputs weighed by their frequencies.

The reconstruction error $R^{(0)}$ for $t=0$ is:
\eq{
R^{(0)}&=1/l\cdot(k-(1+k)/l)^2+1/l\cdot(1-(1+k)/l)^2\\
&+(1-2/l)\cdot(0-(1+k)/l)^2\\
&=(-1 - 2 k - k^2 + l + k^2 l)/l^2\\
&=\frac{-(k+1)^2 + l(1 + k^2)}{l^2}\\
&=\Omega(k^2/l) \teb{ since $k>1$ and $l\geq 2$}
}
Let us consider the limit $t\rightarrow \infty$: Here, $R^{(t \rightarrow \infty)}_{0}=0$, which is optimal. For attributes $j$ with $x_j=0$ we have that $\lim_{t \rightarrow \infty} q^{(t)}(y''\neq y|C^{y}_1)-q^{(t)}(y'\neq y|C^{y}_k) =0$ and thus, the nominator in $g_j$ (Eq. \ref{eq:gj}) is 0 and, in turn, $g_j$ as well, yielding no error. For $x_j=v$ with $v \in \{1,k\}$ the outputs are  converging towards 1. Since the outputs are identical the best reconstruction is the average, i.e., $(k+1)/2$, yielding an error of 
\eq{
&R^{(t \rightarrow \infty),opt}=R^{(t \rightarrow \infty),opt}_{0}+R^{(t \rightarrow \infty),opt}_{1}+R^{(t \rightarrow \infty), opt}_{k}\\
&=0+\frac{1}{l}((1-(k+1)/2)^2+(k-(k+1)/2)^2)\\
&= \frac{(k-1)^2}{2l}
}
The reconstruction error using $g_j$ is a factor 2 larger:
\eq{
&R^{(t \rightarrow \infty)}_{0}=R_{k}=0\\
&R^{(t \rightarrow \infty)}_{1}=1/l(1-k)^2\\
&R^{(t \rightarrow \infty)}=R^{(t \rightarrow \infty)}_0+R^{(t \rightarrow \infty)}_1+R^{(t \rightarrow \infty)}_k=\frac{(k-1)^2}{l}
}
\end{proof}

\begin{thm}\label{le:rec4}  
The total reconstruction error $R=O(k^2/l^2)$ for $1/l< f<c_f/l$ and $R_{0,C^{y'\neq y}_1}=\Omega(k/l)$ for $c_f/l<f<1$ and $R=O(1/l)$ for $1\leq f \leq \frac{\log l -1}{k}$  and $\Theta(k^2/l)$  for $f > \frac{\log l -1}{k}$ for an arbitrary constant $c_f>1$.
\end{thm}

\begin{proof}
We compute the total reconstruction error $R$:
\eq{
&R:= (1-\frac{1}{2l})(R_{0,C^{y'}_1}+R_{0,C^{y'}_k})+\frac{1}{l}(R_{1}+R_{k}) \teb{ Eq. \ref{def:brecLoss}} \\
&= (1-\frac{1}{2l})R_{0,C^{y'}_1}+\frac{1}{l}R_{1} \\
& \teb{ since  $R_{0,C^{y'}_k}=0, R_k=0$ by Def. of $g_j$ (see Sec. \ref{sec:zerorec})} \label{eq:rectot}
}
In the derivation we use repeatedly Eq. \ref{eq:rectot} and Theorems \ref{thm:R2}, \ref{thm:R1} and \ref{thm:R0}.
\eq{
&\text{Assume: } 1/l< f<c/l \\
&R= (1-\frac{1}{2l})O(k^2/l^2)+\frac{1}{l}O(k^4/l^2) \\
&= O(k^2/l^2+k^4/l^3)=O(k^2/l^2) \\
&\text{Assume: } c/l< f<1 \\
&R=(1-\frac{1}{2l})O(k^2/l^2)+\frac{1}{l}\Omega(k)  \\
&=\Omega(k/l)\\
&\text{Assume: } 1\leq f \leq \frac{\log l -1}{k} \\
&R= O(1/l) \\
&\text{Assume: } f > \frac{\log l -1}{k} \\
&R= \Theta(k^2/l)  
}
\end{proof}

\begin{cor}
The reconstruction error for $g_j$ decreases from $R^{(0)}=\Omega(k^2/l)$ to $R^{(t)}=O(k^2/l^2)$ with $t \in [1,\Theta(\frac{2l(c_f-1)}{\lambda d(k+1)})]$ for an arbitrary constant $c_f>1$ and increases to $R^{(t\rightarrow\infty)}=\Omega(k^2/l)$.
\end{cor}

\begin{proof}
The claims for $R^{(0)}$ and $R^{(t\rightarrow\infty)}$ stem directly from Theorem \ref{le:recBInit}.
For $t\geq 1$, based on Theorem \ref{le:rec4}  we have $R=O(k^2/l^2)$ for $1/l< f<c_f/l$. Using Theorem \ref{thm:boundS} we have that for $f\in[1/l,\log(l)]$ in $t \in \Theta(\frac{2l^2(f-1/l)}{\lambda d(k+1)})$. Substitute $f=c_f/l$ yields $t =\Theta(\frac{2l(c_f-1)}{\lambda d(k+1)})$.
\end{proof}

\begin{thm}
For the approximation error of the reconstruction function $g_j$ of the optimal function $g^{opt}$ holds $||g_j-g^{opt}_j||<2$ .
\end{thm}

\begin{proof}
We show that any other linear reconstruction function leads to the same asymptotic error. We can obtain any reconstruction function by changing the height $h$ of a point $(q,h)$ of the two points  through which we fit $g_j$.  We show that altered points can reduce the most dominant part of the four error terms in the sum $R=R_{k}+R_{1}+R_{0,C^{y'}_1}+R_{0,C^{y'}_k}$ only up to a constant factor. Once two error terms are equal, reducing either of them further, increases the other, i.e., asymptotically the error cannot be further reduced. 
Assume we change $h=k$ in $(q(y|C^y_k),h=k)$. For $h=k$ there is no error by definition of $g_j$, since we fit $g_j$ using this point. Say we choose $h \notin [k-1/4,k+1/4]$ then for $g_j(q(y|C^y_k))$ and the error $R_k$ we get
\eq{
&g_j(q(y|C^y_k))=\frac{h}{q(y|C^y_k)-q(y|C^{y'\neq y}_k)}\cdot (q(y|C^y_k)-q(y|C^{y'\neq y}_k))=h \\
& \teb{ Using Eq \ref{eq:gj}}\\
& R_k= (k-h)^2 \geq 1/16 = O(1)
}
Thus, we cannot change the height $h$ by 1/4 or more, since otherwise $R_{k}$ is the (asymptotically) the dominant error. 
Furthermore, changing by 1/4 does not help in reducing other errors. Varying the height by $1/4$ has only a modest impact on the slope in $g_j$, i.e., $\frac{h}{q(y|C^y_k)-q(y|C^{y'\neq y}_k)}$. That is, the slope is changed by a factor of at most $7/4\leq (k\pm 1/4)/k\leq 9/4$ for $k\geq 2$. Changing the slope in this interval has no asymptotic impact on $R_1$ as can be seen by replacing $k$ with $ck$ for $c \in [7/4,9/4]$ in the equations used to compute $R_1$ for different ranges Eq. \ref{eq:reck0},\ref{eq:reck1},\ref{eq:reck3} and\ref{eq:recf2}.

Let us examine changing $h=0$ in the other point used for fitting $g_j$, i.e., $(q(y|C^{y'\neq y}_k,h=0)$. In an analogous manner, if we choose $h \notin [-1/l,1/l]$ the error $R_{0,C^{y'\neq y}_k}$ for samples $C^{y'\neq y}_k$ increases to $(0-1/l)^2=O(1/l^2)$. Thus, let us investigate using $(q(y|C^{y'\neq y}_k,s)$ for a value $s\in [-1/l,1/l]$. For ease of analysis, let us assume that we shift both points by that value, i.e. we add an offset $s$ to the line by $g_j$. That is, we use points  $(q(y|C^{y'\neq y}_k,s)$ and $(q(y|C^y_k),k+s)$. Note that we have already shown that changing the height of $(q(y|C^y_k),k)$ by $1\gg s$ has no impact, so we can certainly shift it by $s$ without any asymptotic change. Thus, we compute $R_1$ (Eq. \ref{eq:reck0},\ref{eq:reck1},\ref{eq:reck3} and\ref{eq:recf2}), when adding an offset of $s$.
\eq{
&R_{1}=(1-k(\frac{e-1}{e^k-1})-O(kf/l)+s )^2 \teb{ using Eq \ref{eq:reck1}}\\
&=\Omega(k)\\
&R_{1}=(1-  \Theta(k(1/\sqrt{l}+s))^2  \teb{ using Eq\ref{eq:reck3}}\\
&=O(1)\\
&R_{1}=(1- \frac{kc}{e^{f(k-1)}}-ks-O(k/l))^2 = O(1) \teb{ using Eq \ref{eq:recf2}}
}
Comparing these errors with shifting by $s\in [-1/l,1/l]$ to the original ones (Eq. \ref{eq:reck1},\ref{eq:reck3} and\ref{eq:recf2}), it becomes apparent that no asymptotic change took place. 

Still, we need to assess Eq. \ref{eq:reck0}. Note that $R_{0,C^{y'\neq y}_k}=(0-s)^2=O(s^2)$, while $R_1$ is:

\eq{
&R_{1}=(1-k(\frac{1}{k}+O(k/l)+s) )^2 \teb{ using Eq\ref{eq:reck0}}\\
&= (O(k/l)+ks)^2
}
Here, in principle $(O(k/l)+ks)$ could cancel, i.e. $R_1$ could become 0, while without shifting $R_1 = O(k^4/l^2)$ (see Eq. \ref{eq:reck0}). However, since $R_{0,C^{y'\neq y}_k}=O(s^2)$ reducing $R_1$ increases $R_{0,C^{y'\neq y}_k}=O(s^2)$ leading to no possibility for an asymptotic change.\footnote{As a side remark, in this interval for $f$, i.e., for $1/l<f<c_f/l$, smaller $R_1$ would not hamper our overall claim of the existence of phases, since here the error is shown to be smallest (among all phases) using $g_j$, thus, if it is even smaller for the optimal reconstruction using $g^{opt}$, the statement still hold on a qualitative level.}
\end{proof}


\subsection{Random initialization} \label{sec:ra}
The prior analysis assumed deterministic initialization. Next, we discuss the impact of random initialization. 
We show that outputs of a network with randomly initialized weights are concentrated around the mean. In particular, the initial noise for common initialization schemes, e.g. using  $w \sim N(1,1/d)$ (see Eq. \ref{ass:initw}) is much less than the changes done to weights during early training. Thus, the noise has some impact on reconstruction very early in training but gets less relevant as changes due to learning become dominant. A tail bound can be used to show that the sum of weights $S_{y,y'}$ is tightly concentrated. 

We use the next lemma to derive our tail bounds.

\begin{lem} \label{lem:gau}
For a Gaussian random variable $X \in N(0,\sigma)$ holds for $a>0$ that $e^{-(a+1)^2/(2\sigma^2)}\leq P(X\geq a) \leq  e^{-a^2/(2\sigma^2)}$
\end{lem} 

\begin{proof}
We have $p(X=x)= \frac{1}{\sigma\sqrt{2\pi}}e^{-x^2/(2\sigma^2)}$ and
\eq{
P(X\geq a) &= \int_{a}^{\infty} \frac{1}{\sigma\sqrt{2\pi}}e^{-x^2/(2\sigma^2)} \\
= \int_{0}^{\infty} \frac{1}{\sigma\sqrt{2\pi}}e^{-(x+a)^2/(2\sigma^2)}
}
We lower and upper bound the last expression, we use the fact integral of $p(X=x)$ must by 1 (since $X$ is a distribution), i.e., 
\eq{
\int_{-\infty}^{\infty} \frac{1}{\sigma\sqrt{2\pi}}e^{-x^2/(2\sigma^2)}=1} and thus due to symmetry 
\eq{\int_{0}^{\infty} \frac{1}{\sigma\sqrt{2\pi}}e^{-x^2/(2\sigma^2)}=1/2}
Upper bound:
\eq{
P(X\geq a)\leq \int_{0}^{\infty} \frac{1}{\sigma\sqrt{2\pi}}e^{-(x^2+a^2)/(2\sigma^2)}
=1/2e^{-(a^2)/(2\sigma^2)}
}
Lower bound:
\eq{
&P(X\leq a)\leq \int_{0}^{\infty} \frac{1}{\sigma\sqrt{2\pi}}e^{-(x^2+2a+a^2)/(2\sigma^2)}\\
&=1/2e^{-(a^2+2a)/(2\sigma^2)}\\
&\geq 1/2e^{-(a+1)^2/(2\sigma^2)}
}
\end{proof}

\begin{thm} \label{thm:init}
After initialization, with probability $1-1/l^{2}$ holds for all classes $y$ that $S_{y,y'} \in [-a,a]$ with $a:=-2\sqrt{\log{l}}/l$
\end{thm} 

\begin{proof}
We first show that $S_y$ is concentrated around its mean by deriving a distribution of $S_y$ being the sum of (products of) random variables and using tail bounds. Initially, the term $S_{y}:=\sum_{i \in A_{y}} w_i$ is a sum of $|A_y|=d/l$ independent normally distributed variables $w_i \sim N(0,1/d)$ (Def. \ref{ass:initw})
Thus, $S_y$ is also normally distributed with variance (due to linearity of the variance):
\eq{Var(S_y)=Var(\sum_{i \in A_{y}} w_i)=\sum_{i \in A_{y}} Var(w_i)=d/l\cdot 1/d=1/l}
Thus, $S_y \sim N(0,1/l)$.
We use the tail bounds from Lemma \ref{lem:gau} $p(S_{y} \geq a) \leq e^{a^2/(2\sigma^2)}/2$  with $a=2\sqrt{\log{l}}/l$ yielding
\eq{p(S_{y} \geq a) \leq e^{-(4\log l)/l^2/(1/l)^2}/2=1/(2l^4)}
The probability that the bound on $S_y$ holds for all $l$ classes, i.e., for all $y\in [0,l-1]$, is \eq{(1-1/(l^4))^l \leq 1-1/l^{2}}
\end{proof}

Based on the theorem, initially probabilities are tightly concentrated around $1/l$ similar to the deterministic case. We can see this by looking at the initial variance of $S^{(0)}_y$:
$Var(e^{S^{(0)}_y}) \approx Var(e^{w^{(0)}_{y,i}})^{d/l}$.
Since $w^{(0)}_{y,i}$ is normally distributed, $e^{w^{(0)}_{y,i}}$ has a log-normal distribution with variance $(e^{\sigma^2}-1)e^{2\mu + \sigma^2}=e^{2\sigma^2}-1$. Since $\sigma=1/d$ is close to 0 (as $d$ is large), $e^{2\sigma^2}$ will be close to 1 and the variance will be small, i.e., close to 0. Given the concentration of sum of weights, the changes to weights are already early on of similar magnitude for random and deterministic initialization.

\begin{figure}[!htb]
\centering{\centerline{\includegraphics[width=0.5\textwidth]{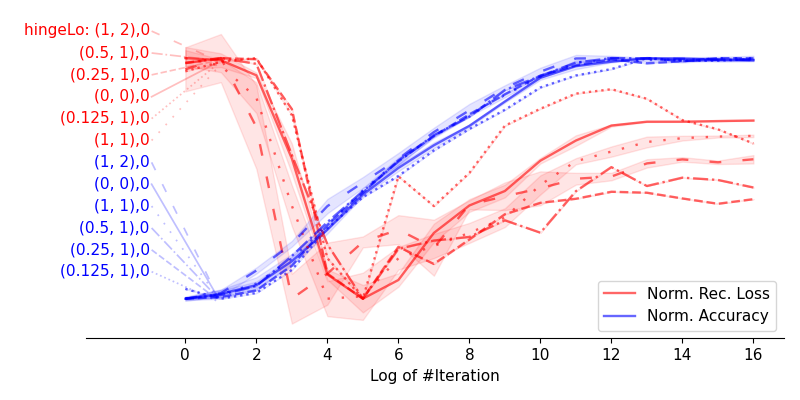}} 
\caption{Hinge Loss }\label{fig:hinge}
}
\end{figure}

\begin{figure}[!htb]
\centering{\centerline{\includegraphics[width=0.5\textwidth]{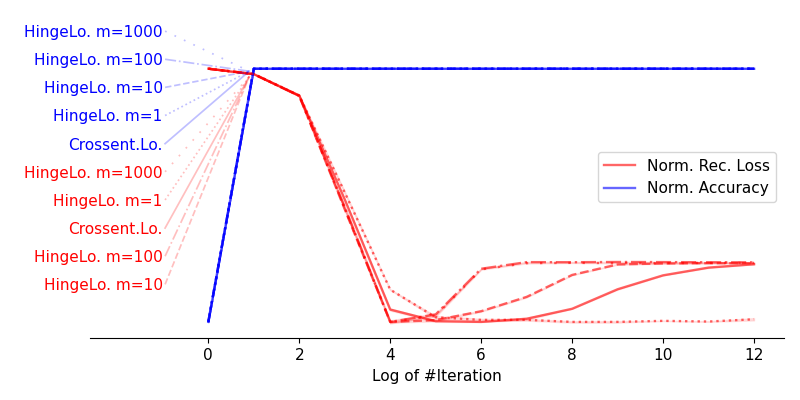}} 
\caption{Hinge Loss for linear layer network}\label{fig:linhinge}
}
\end{figure}

\subsection{Hinge losses}
See Figures \ref{fig:linhinge} and \ref{fig:hinge}.

\subsection{Other datasets and classifiers showing existence of phases} 
In Figures \ref{fig:metMuFa3},\ref{fig:metMuFa4}, \ref{fig:metMuFa5} and \ref{fig:metMuFa6} we show additional combinations not in the main paper for classifiers, i.e., F0, VGG-16, ResNet-10, and datasets, i.e., CIFAR-10 and CIFAR-100, and FashionMNIST. 

\begin{figure}
  \centering  
  \includegraphics[width=0.45\textwidth]{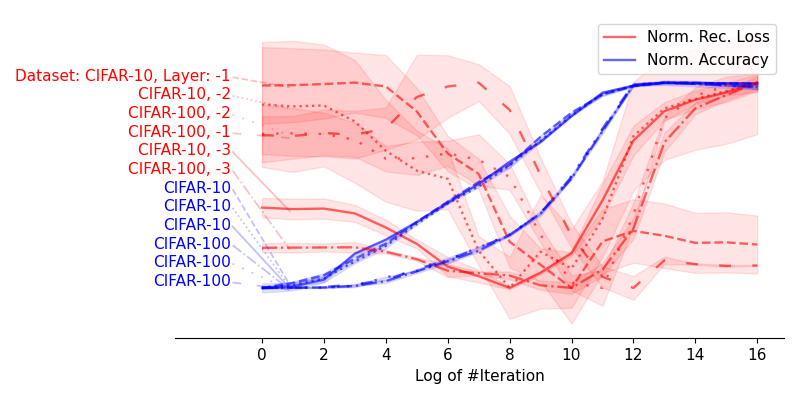} 
  \caption{Normalized accuracy and reconstruction loss for the CIFAR-10 and CIFAR-100 datasets for ResNet-10} \label{fig:metMuFa3}
\end{figure}

\begin{figure}
  \centering  
  \includegraphics[width=0.45\textwidth]{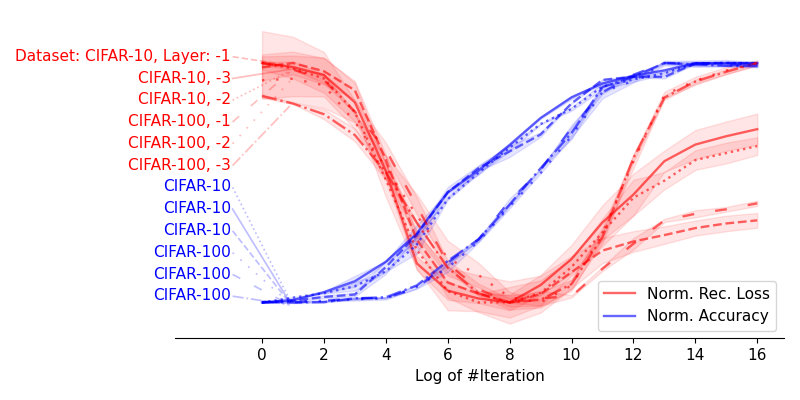} 
  \caption{Normalized accuracy and reconstruction loss for multiple classifiers for the CIFAR-10 and CIFAR-100 datasets for VGG-16} \label{fig:metMuFa4}
\end{figure}

\begin{figure}
  \centering  
  \includegraphics[width=0.45\textwidth]{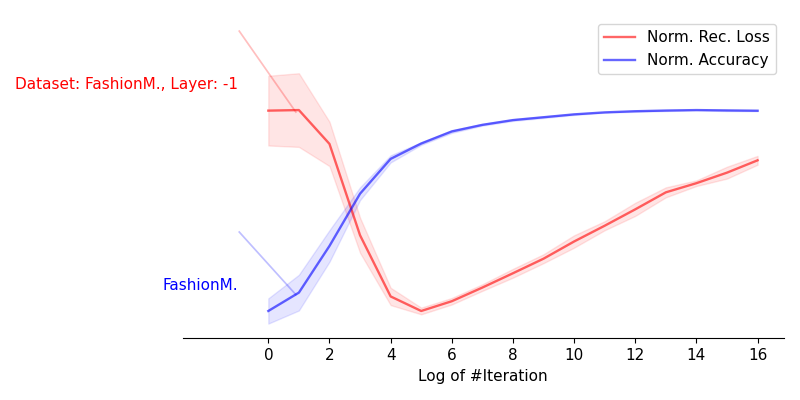} 
  \caption{Normalized accuracy and reconstruction loss  for the FashionMNIST dataset for VGG-16} \label{fig:metMuFa5}
\end{figure}

\begin{figure}
  \centering  
  \includegraphics[width=0.45\textwidth]{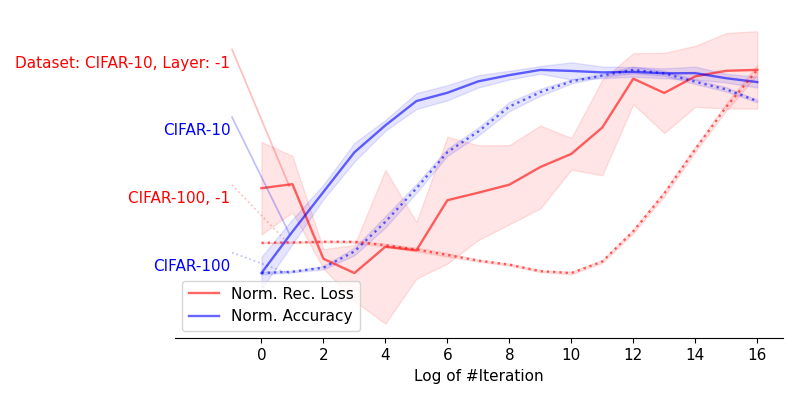} 
  \caption{Normalized accuracy and reconstruction loss  for the CIFAR-10 and CIFAR-100 dataset for classifier F0} \label{fig:metMuFa6}
\end{figure}

\subsection{Other datasets and classifiers showing transfer learning}
Figures \ref{fig:trans3} to \ref{fig:trans14} show additional experiments showing that performance on downstream tasks tends to be optimal when the pre-training of a classifier has not yet converged, i.e., its predictive performance could still be significantly improved through further training. This observation is well-visible for most almost all tasks, i.e., for 10 out of 12 experiments (Note, the two plots from the main manuscript are not shown again). The only exception is MNIST to FashionMNIST, here the effect is only very weak for both Resnet and VGG. This could be since MNIST is a very simple dataset, where not much information is discarded in the later phases.

\begin{figure}[!htb]
\centering{\centerline{\includegraphics[width=0.45\textwidth]{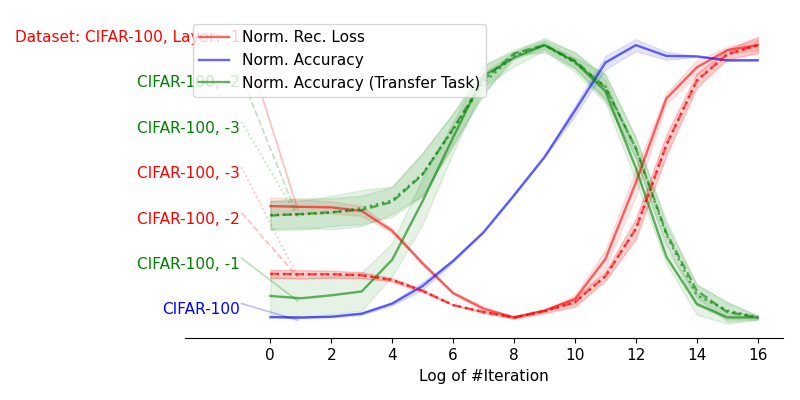}}} 
\vspace{-6pt}
\caption{Transfer Learning of VGG to predict Cifar-100 classes to predict the most dominant color channel. Lines are normalized.}\label{fig:trans3}
\vspace{-6pt}
\end{figure}

\begin{figure}[!htb]
\centering{\centerline{\includegraphics[width=0.45\textwidth]{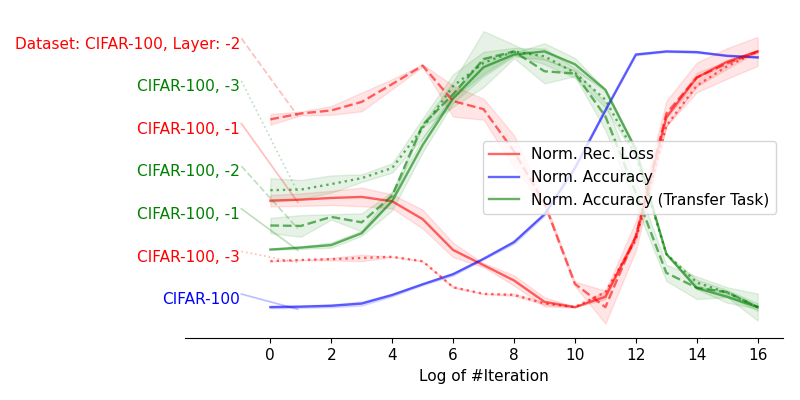}}} 
\vspace{-6pt}
\caption{Transfer Learning of Resnet to predict Cifar-100 classes to predict the most dominant color channel. Lines are normalized.}\label{fig:trans4}
\vspace{-6pt}
\end{figure}

\begin{figure}[!htb]
\centering{\centerline{\includegraphics[width=0.45\textwidth]{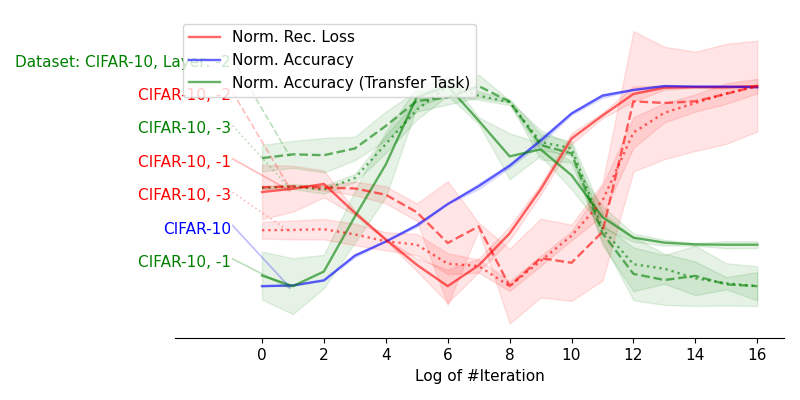}}} 
\vspace{-6pt}
\caption{Transfer Learning of Resnet to predict Cifar-10 classes to predict the most dominant color channel. Lines are normalized.}\label{fig:trans5}
\vspace{-6pt}
\end{figure}


  \begin{figure}[!htb]
    \centering{\centerline{\includegraphics[width=0.45\textwidth]{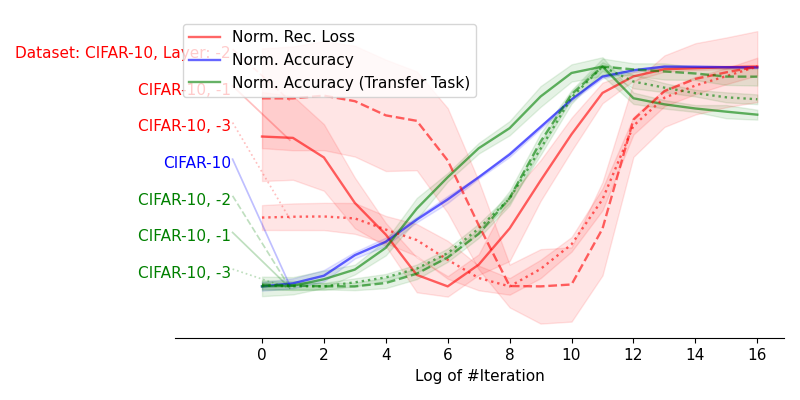}}} \vspace{-6pt}
    \caption{Transfer Learning of a Resnet from Cifar-10 to Cifar-100. Lines are normalized.}\label{fig:trans8}  \vspace{-6pt}\end{figure}

  \begin{figure}[!htb]
    \centering{\centerline{\includegraphics[width=0.45\textwidth]{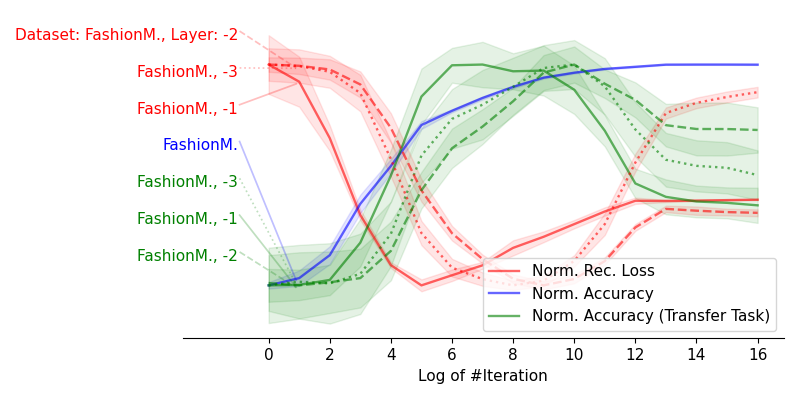}}} \vspace{-6pt}
    \caption{Transfer Learning of a Resnet from FashionMNIST to MNIST. Lines are normalized.}\label{fig:trans9}  \vspace{-6pt}\end{figure}

  \begin{figure}[!htb]
    \centering{\centerline{\includegraphics[width=0.45\textwidth]{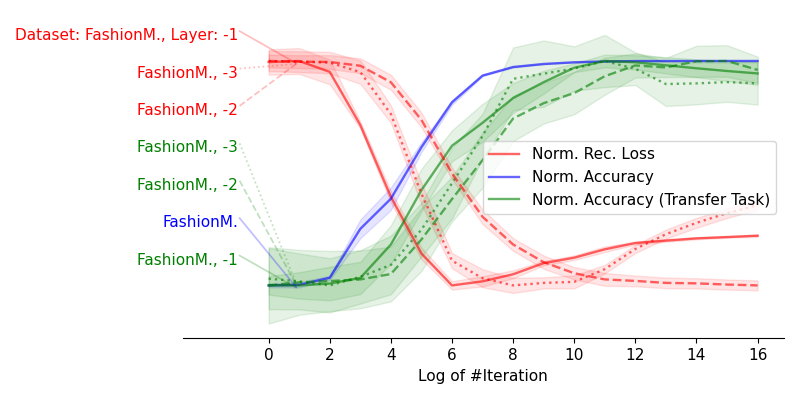}}} \vspace{-6pt}
    \caption{Transfer Learning of a Resnet from MNIST to FashionMNIST. Lines are normalized.}\label{fig:trans10}  \vspace{-6pt}\end{figure}

  \begin{figure}[!htb]
    \centering{\centerline{\includegraphics[width=0.45\textwidth]{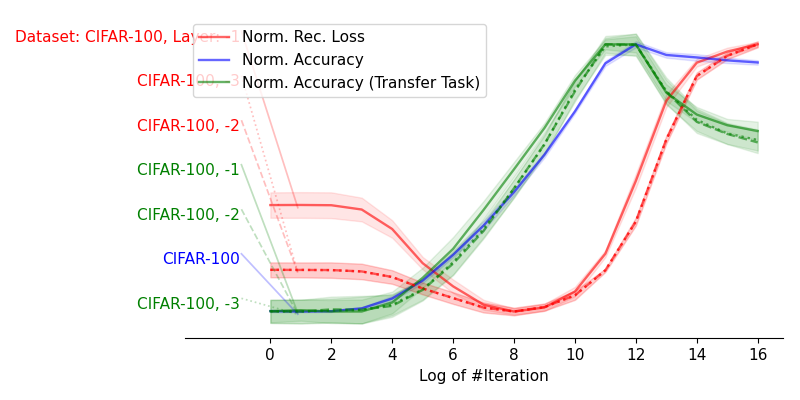}}} \vspace{-6pt}
    \caption{Transfer Learning of a VGG from Cifar-100 to Cifar-10. Lines are normalized.}\label{fig:trans11}  \vspace{-6pt}\end{figure}

  \begin{figure}[!htb]
    \centering{\centerline{\includegraphics[width=0.45\textwidth]{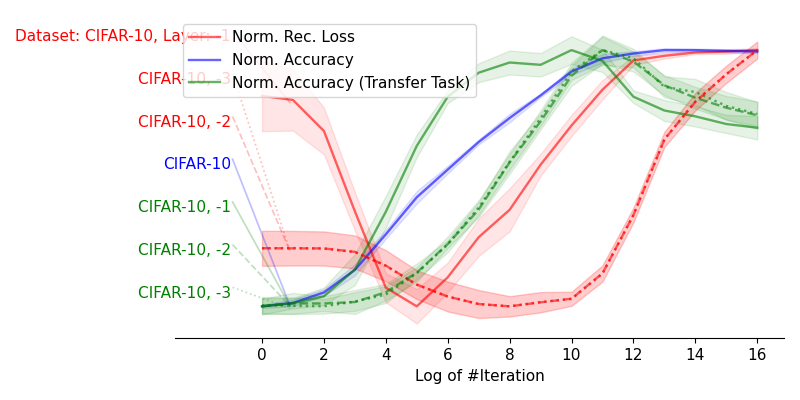}}} \vspace{-6pt}
    \caption{Transfer Learning of a VGG from Cifar-10 to Cifar-100. Lines are normalized.}\label{fig:trans12}  \vspace{-6pt}\end{figure}

  \begin{figure}[!htb]
    \centering{\centerline{\includegraphics[width=0.45\textwidth]{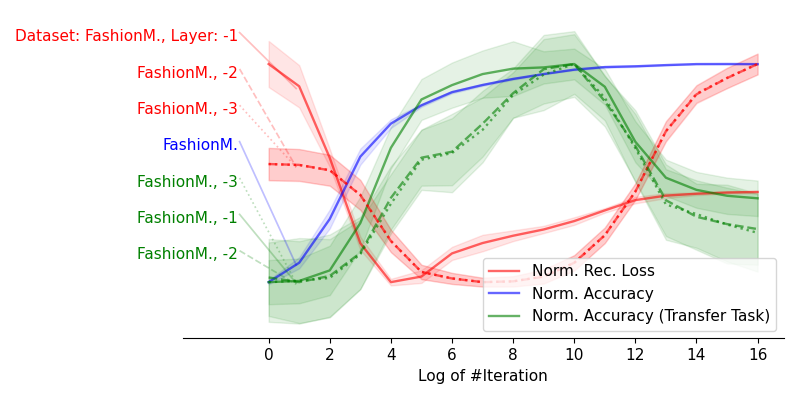}}} \vspace{-6pt}
    \caption{Transfer Learning of a VGG from FashionMNIST to MNIST. Lines are normalized.}\label{fig:trans13}  \vspace{-6pt}\end{figure}

  \begin{figure}[!htb]
    \centering{\centerline{\includegraphics[width=0.45\textwidth]{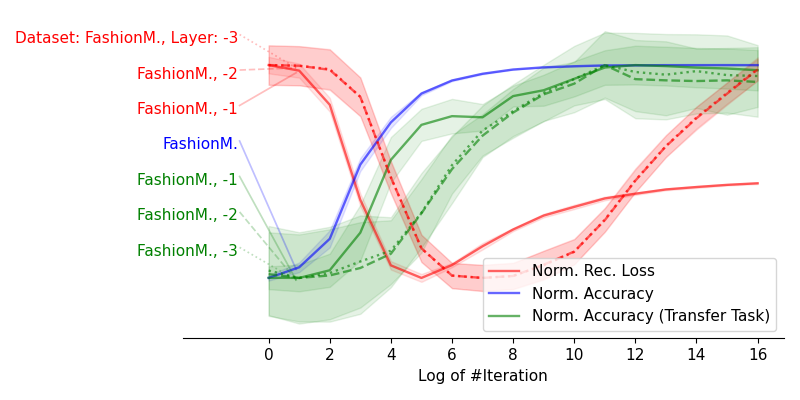}}} \vspace{-6pt}
    \caption{Transfer Learning of a VGG from MNIST to FashionMNIST. Lines are normalized.}\label{fig:trans14}  \vspace{-6pt}\end{figure}
\fi

\end{document}